\documentclass{article}

\usepackage{microtype}
\usepackage{graphicx}
\usepackage{booktabs} 
\usepackage{import}
\usepackage{xcolor}
\usepackage{float}
\usepackage{hyperref}

\usepackage{./icml_formating/icml2021}

\icmltitlerunning{Ensemble Bootstrapping for Q-Learning}
\usepackage[utf8]{inputenc}
\usepackage{amsfonts,amsthm}
\usepackage{amsmath, esdiff, amsbsy}
\usepackage{amssymb}
\usepackage{mathtools}
\usepackage{cancel}
\usepackage{graphicx}
\usepackage{nicefrac}
\usepackage{caption}
\usepackage{subcaption}
\usepackage{multirow}
\usepackage{diagbox}
\usepackage{thmtools,thm-restate}
\usepackage[inline]{enumitem}
\usepackage[capitalize]{cleveref}

\def\argmax{\mathop{\mathrm{argmax}}}

\usepackage{mathtools}
\DeclarePairedDelimiter\abs{\lvert}{\rvert}
\DeclarePairedDelimiter\norm{\lVert}{\rVert}

\DeclarePairedDelimiter\ceil{\lceil}{\rceil}

\newcommand{\calig}[1]{\mathcal{#1}}

\newcommand{\brac}[1]{\left(#1\right)}
\newcommand{\sbrac}[1]{\left[#1\right]}
\newcommand{\cbrac}[1]{\left\{#1\right\}}

\newcommand{\A}{\calig{A}}
\newcommand{\B}{\calig{B}}

\newcommand{\M}{\calig{M}}
\newcommand{\N}{\calig{N}}

\newcommand{\Scal}{\calig{S}}

\newcommand{\U}{\calig{U}}

\newcommand{\dsE}{\mathbb{E}}
\newcommand{\dsR}{\mathbb{R}}

\newcommand{\alg}{EBQL}
\newcommand{\algfull}{Ensemble Bootstrapped Q-Learning}

\newcommand{\SNR}{\mathrm{SNR}}
\newcommand{\var}{\mathrm{var}}
\newcommand{\bias}{\mathrm{bias}}
\newcommand{\MSE}{\mathrm{MSE}}
\newcommand{\EE}{\text{EE}}
\newcommand{\DE}{\text{DE}}
\newcommand{\SE}{\text{SE}}
\newcommand{\WDE}{\text{W-DE}}

\usepackage{selectp}

\begin{document}
\twocolumn[
\icmltitle{Ensemble Bootstrapping for Q-Learning}

\begin{icmlauthorlist}
\icmlauthor{Oren Peer}{Tech}
\icmlauthor{Chen Tessler}{Tech}
\icmlauthor{Nadav Merlis}{Tech}
\icmlauthor{Ron Meir}{Tech}
\end{icmlauthorlist}
\icmlaffiliation{Tech}{Viterbi Faculty of Electrical Engineering, Technion Institute of Technology, Haifa, Israel}
\icmlcorrespondingauthor{Oren Peer}{\mbox{orenpeer@campus.technion.ac.il}}
%

\icmlkeywords{Machine Learning, ICML}
\vskip 0.3in
]
\printAffiliationsAndNotice{} 


\begin{abstract}
Q-learning (QL), a common reinforcement learning algorithm, suffers from over-estimation bias due to the maximization term in the optimal Bellman operator. 
This bias may lead to sub-optimal behavior. 
Double-Q-learning tackles this issue by utilizing two estimators, yet results in an under-estimation bias. 
Similar to over-estimation in Q-learning, in certain scenarios, the under-estimation bias may degrade performance.
In this work, we introduce a new bias-reduced algorithm called \algfull{} (\alg), a natural extension of Double-Q-learning to ensembles.
We analyze our method both theoretically and empirically. Theoretically, we prove that \alg-like updates yield lower MSE when estimating the maximal mean of a set of independent random variables. Empirically, we show that there exist domains where both over and under-estimation result in sub-optimal performance. 
Finally, We demonstrate the superior performance of a deep RL variant of \alg{} over other deep QL algorithms for a suite of ATARI games.
\end{abstract}

\section{Introduction} \label{Introduction}

In recent years, reinforcement learning (RL) algorithms have impacted a vast range of real-world tasks, such as robotics \citep{andrychowicz2020learning, kober2013reinforcement}, power management 
\citep{xiong2018power}, autonomous control \cite{bellemare2020autonomous}, traffic control \citep{abdulhai2003traffic1, wiering2000multi}, and more \citep{mahmud2018applications,luong2019applications2}. These achievements are possible by the ability of RL agents to learn a behavior policy via interaction with the environment while simultaneously maximizing the obtained reward.


A common family of algorithms in RL is Q-learning \citep[QL]{watkins1992q} based algorithms, which focuses on learning the value-function. 
The value represents the expected, discounted, reward-to-go that the agent will obtain. In particular, such methods learn the optimal policy via an iterative maximizing bootstrapping procedure. A well-known property in QL is that this procedure results in a positive estimation bias. \citet{hasselt2010double} and \citet{van2016deep} argue that over-estimating the real value may negatively impact the performance of the obtained policy. 
They propose Double Q-learning (DQL) and show that, as opposed to QL, it has a negative estimation bias.

The harmful effects of \emph{any} bias on value-learning are widely known.  
Notably, the deadly triad \citep{sutton2018reinforcement,van2018deep} attributes the procedure of bootstrapping with biased values as one of the major causes of divergence when learning value functions. This has been empirically shown in multiple works \citep{van2016deep, van2018deep}. Although this phenomenon is well known, most works focus on avoiding positive-estimation instead of minimizing the bias itself \citep{van2016deep,wang2016dueling,hessel2018rainbow}.

In this work, we tackle the estimation bias in an underestimation setting. We begin, similar to \citet{hasselt2010double}, by analyzing the estimation bias when observing a set of i.i.d.\ random variables (RVs). We consider three schemes for estimating the maximal mean of these RVs. \textit{(i)} The \emph{single-estimator}, corresponding to a QL-like estimation scheme. \textit{(ii)} The \emph{double-estimator}, corresponding to DQL, an estimation scheme that splits the samples into two equal-sized sets: one for determining the index of the RV with the maximal mean and the other for estimating the mean of the selected RV. 
\textit{(iii)} The \emph{ensemble-estimator}, our proposed method, splits the data into $K$ sets. Then, a single set is used for estimating the RV with the maximal-mean whereas all other sets are used for estimating the mean of the selected RV. 

We start by showing that the ensemble estimator can be seen as a double estimator that unequally splits the samples (\cref{prop:proxy}). This enables, for instance, less focus on maximal-index identification and more on mean estimation. We prove that the ensemble estimator under-estimates the maximal mean (\cref{lemma:EE underestimation}). We also show, both theoretically (\cref{prop:split_ratio}) and empirically (\cref{fig:MSE as func of r}), that in order to reduce the magnitude of the estimation error, a double estimator with equally-distributed samples is sub-optimal.

Following this theoretical analysis, we extend the ensemble estimator to the multi-step reinforcement learning case. We call this extension \algfull{} (\alg). We analyze \alg{} in a tabular meta chain MDP and on a set of ATARI environments using deep neural networks. In the tabular case, we show that as the ensemble size grows, \alg{} minimizes the magnitude of the estimation bias. 
Moreover, when coupled with deep neural networks, we observe that \alg{} obtains superior performance on a set of ATARI domains when compared to Deep Q-Networks (DQN) and Double-DQN (DDQN).


\textbf{Our contributions are as follows:}
\begin{itemize}
    \item We Analyze the problem of estimating the maximum expectation over independent random variables. We prove that the estimation mean-squared-error (MSE) can be reduced using the Ensemble Estimator. In addition, we show that obtaining the minimal MSE requires utilizing more than two ensemble members.
    \item Drawing inspiration from the above, we introduce \algfull{} (\alg) and show that it reduces the bootstrapping estimation bias.
 %
%
    \item We show that \alg{} is superior to both Q-learning and double Q-learning in both a tabular setting and when coupled with deep neural networks (ATARI).
\end{itemize}

\section{Preliminaries} \label{Preliminaries} 
\subsection{Model Free Reinforcement Learning}

A reinforcement learning \cite{sutton2018reinforcement} agent
faces a sequential decision-making problem in some unknown or partially known environment. We focus on environments in which the action space is discrete.
The agent interacts with the environment as follows. At each time $t$, the agent observes the environment state $s_t \!\in\! \Scal$ and selects an action $a_t \!\in\! \A$. Then, it receives a scalar reward $r_t \!\in\! \dsR$ and transitions to the next state $s_{t+1} \!\in\! S$ according to some (unknown) transition kernel $P(s_{t+1} | s_t, a_t)$. 
A policy $\pi:\Scal\to\A$, is a function that determines which action the agent should take at each state and the sampled performance of a policy, a random variable starting from state $s$, is denoted by
\begin{align}
    R^\pi (s) = \sum_{t=0^\infty} \gamma^t r_t | s_0 = s, a \sim \pi (s_t) \, \label{eq:reward-to-go}.
\end{align}
where $\gamma \in [0, 1)$ is the discount factor, which determines how myopic the agent should behave. 
As $R^\pi$ is a random variable, in RL we are often interested in the expected performance, also known as the policy value, 
\begin{align*}
    v^\pi (s) = \dsE^\pi \sbrac{\sum_{t=0}^\infty \gamma^{t} r(s_{t}) | s_0 = s}  \, .
\end{align*}
The goal of the agent is to learn an optimal policy $\pi^*\in\argmax_\pi v^\pi(s)$. Importantly, there exists a policy $\pi^*$ such that $v^{\pi^*}(s)=v^*(s)$ for all $s\in\Scal$ \citep{sutton2018reinforcement}. 

\subsection{Q-Learning}
One of the most prominent algorithms for learning $\pi^*$ is Q-learning \citep[QL]{watkins1992q}. Q-learning learns the Q-function, which is defined as follows
\begin{equation*}
    Q^\pi(s,a) = \dsE^\pi \sbrac{\sum_{t=0}^\infty \gamma^tr_t(s_t,a_t) |s_0=s, a_0=a} \, .
\end{equation*}
Similar to the optimal value function, the optimal Q-function is denoted by $Q^*(s,a)= \max_\pi Q^\pi (s,a), \ \forall (s,a) \!\in\! \Scal\!\times\! \A$.
Given $Q^*$, the optimal policy $\pi^*$ can obtained by using the greedy operator $\pi^*(s) = \argmax_a Q^*(s,a)\ \forall s \in \Scal$. 

QL is an off-policy algorithm that learns $Q^*$ using transition tuples $\brac{s, a, r, s'}$. Here, $s'$ is the state that is observed after performing action $a$ in state $s$ and receiving reward $r$. Specifically, to do so, QL iteratively applies the optimal Bellman equation \citep{bellman1957markovian}:
\begin{equation}\label{eqn: bellman}
    Q(s,a) \leftarrow (1-\alpha)Q(s,a) + \alpha (r_t + \gamma \max_{a'}Q(s',a')) \, .
\end{equation}
Under appropriate conditions, QL asymptotically converges almost surely to an optimal fixed-point solution \cite{tsitsiklis1994asynchronous}, i.e. $Q_t(s,a) \xrightarrow{t\rightarrow \infty} Q^*(s,a) \ \forall s\in\Scal,a\in\A$.

A known phenomenon in QL is the over-estimation of the Q-function. The Q-function is a random variable, and the optimal Bellman operator, \cref{eqn: bellman}, selects the maximal value over all actions at each state. Such estimation schemes are known to be positively biased \citep{smith2006optimizer}.

\subsection{Double Q-Learning}
A major issue with QL is that it overestimates the true Q values. \citet{hasselt2010double, van2016deep} show that, in some tasks, this overestimation can lead to slow convergence and poor performance. 
To avoid overestimation, Double Q-learning \citep[DQL]{hasselt2010double} was introduced. DQL combines two Q-estimators: $Q^A$ and $Q^B$, each updated over a unique subset of gathered experience. 
For any two estimators $A$ and $B$, the update step of estimator $A$ is defined as (full algorithm in the appendix \cref{alg:DQL}):
\begin{align}
    & \hat{a}^*_A = \argmax_{a'} Q^A(s_{t+1}, a')\nonumber\\
    & Q^A(s_t,a_t) \gets (1-\alpha_t)Q^A(s_t,a_t)
    \nonumber\\
    &\qquad\qquad\qquad +\alpha_t \brac{r_t + \gamma Q^B\left(s_{t+1},\hat{a}^*_A\right)}. \label{eq:DQL update}
\end{align}
Importantly, \citet{hasselt2010double} proved that DQL, as opposed to QL, leads to underestimation of the Q-values.

\section{Related Work} \label{related work}

\textbf{Q-learning stability.} Over the last three decades, QL has inspired a large variety of algorithms and improvements, both in the tabular setting \citep{strehl2006pac, kearns1999finite, ernst2005tree, azar2011speedy} and in the function approximation-based settings \citep{schaul2015prioritized,bellemare2017distributional,hessel2018rainbow,jin2018q,badia2020agent57}. 
Many improvements focus on the estimation bias of the learned Q-values, first identified by \citet{thrun1993issues}, and on minimizing the variance of the target approximation-error \citep{anschel2017averaged}.
Addressing this bias is an ongoing effort in other fields as well, such as economics and statistics \cite{smith2006optimizer, thaler2012winner}.
Some algorithms \citep{hasselt2010double, zhang2017weighted} tackle the QL estimator inherent bias by using double estimators, resulting in a \textit{negative} bias.

RL algorithms that use neural-network-based function approximators (deep RL) are known to be susceptible to overestimation and oftentimes suffer from poor performance in challenging environments.
Specifically, estimation bias, a member of the \textit{deadly triad} \citep{van2018deep}, is considered as one of the main reasons for the divergence of QL-based deep-RL algorithms.



In our work, we tackle the estimation bias by reducing the MSE of the next-state Q-values. While \alg{} is negatively biased, we show that it better balances the bias-variance of the estimator. We also show that the bias magnitude of \alg{} is governed by the ensemble size and reduces as the ensemble size grows.

\textbf{Ensembles.} 
One of the most exciting applications of ensembles in RL \citep{osband2018randomized, lee2020sunrise} is to improve exploration and data collection. These methods can be seen as a natural extension of Thompson sampling-like methods to deep RL. 
The focus of this work complements their achievements. While they consider how an ensemble of estimators can improve the learning process, we focus on how to better train the estimators themselves.

%
\section{Estimating the Maximum Expected Value}
As mentioned in \cref{Preliminaries}, QL and DQL display opposing behaviors of over- and under-estimation of the Q-values. To better understand these issues, and our contribution, we take a step back from RL and statistically analyze the different estimators on i.i.d.\ samples.

Consider the problem of estimating the \textit{maximal expected value} of $m$ independent random variables $\brac{X_1, \dots ,X_m}\triangleq X$, with means  $\mu \!=\! \brac{\mu_1, \dots ,\mu_m}$ and standard deviations $\sigma \!=\! \brac{\sigma_1, \dots ,\sigma_m}$. 
Namely, we are interested in estimating
\begin{align} \label{eq:max_exp_X}
    \max_a \dsE \sbrac{X_a} = \max_a \mu_a \triangleq \mu^*\enspace.
\end{align}
The estimation is based on i.i.d.\ samples from the same distribution as $X$. Formally, let $S = \cbrac{S_a}_{a=1}^m$ be a set of samples, where $S_a=\cbrac{S_a(n)}_{n=1}^{N}$ is a subset of $N$ i.i.d.\ samples from the same distribution as $X_a$. Specifically, this implies that $\dsE\sbrac{S_a(n)} = \mu_a$ for all $n\in \sbrac{N_a}$. 
We further assume all sets are mutually independent.

Denote the empirical means of a set $S_a$ by $\hat{\mu}_a(S_a) = 
\frac{1}{N_a}\sum_{n \in [N_a]} S_a(n)$.
This is an unbiased estimator of $\mu_a$, and given sufficiently many samples, it is reasonable to approximate $\hat{\mu}_a(S_a) \approx \mu_a$. Then, a straightforward method to estimate \eqref{eq:max_exp_X} is to use the \textit{Single-Estimator} (SE): 
\begin{align*}
\hat{\mu}_{\SE}^* \triangleq \max_a \hat{\mu}_a(S_a) \approx \max_a \dsE \sbrac{\hat{\mu}_a(S_a)} = \mu^*~.
\end{align*}
As shown in \cite{smith2006optimizer} (and rederived in \cref{proof:se_unbiased} for completeness), this estimator is positively biased.
This overestimation is believed to negatively impact the performance of SE-based algorithms, such as QL. 

To mitigate this effect, \citep{hasselt2010double} introduced the \textit{Double Estimator} (DE). In DE, the samples of each random variable $a\in[m]$ are split into two disjoint, equal-sized subsets $S_a^{(1)}$ and $S_a^{(2)}$, such that $S_a^{(1)} \cup S_a^{(2)} = S_a$ and $S_a^{(1)} \cap S_a^{(2)} = \emptyset$ for all $a\in\sbrac{m}$. For brevity, we denote the empirical mean of the samples in $S_a^{(j)}$ by $\hat{\mu}^{(j)}_a \triangleq\hat{\mu}_a(S^{(j)}_a)$. Then, DE uses a two-phase estimation process:
In the first phase, the \textbf{index} of the variable with the maximal expectation is estimated using the empirical means of $S^{(1)}$, $\hat{a}^* = \argmax_a \hat{\mu}^{(1)}_a$. 
In the second phase, the \textbf{mean} of $X_{\hat{a}^*}$ is estimated using $S^{(2)}_{\hat{a}^*}$, $\hat{\mu}^*_{\DE} = \hat{\mu}^{(2)}_{\hat{a}^*}$. 
The resulting estimator is negatively biased, i.e., $\dsE \sbrac{\hat{\mu}^*_{\DE}} \leq \mu^*$ \citep{hasselt2010double}.

\subsection{The Ensemble Estimator}\label{subsec: ensemble estimator}
In this section, we introduce the \textit{Ensemble Estimator} (EE), a natural generalization of DE to $K$ estimators. 
We take another step forward and ask: 
\begin{center}
    \textit{How can we benefit by using $K$ estimators rather than 2?}
\end{center}

In the DE, two sources are affecting the approximation error. One type of error rises when the wrong index $\hat{a}^*$ is selected (inability to identify the maximal index). The second, when the mean is incorrectly estimated.
While increasing the number of samples used to identify the index reduces the chance of misidentification, it results with fewer samples for mean estimation. Notably, the DE na\"ively allocates the same number of samples to both stages. As we show, the optimal ``split" is not necessarily equal, hence, in an attempt to minimize the total MSE, a good estimator should more carefully tune the number of samples allocated to each stage.

Consider, for example, the case of two independent random variables $(X_1,X_2) \sim \brac{\N(\mu_1, \sigma^2), \N(\mu_2, \sigma^2)}$. 
When the means $\mu_1$ and $\mu_2$ are dissimilar, i.e., $\abs{\mu_1-\mu_2}/\sigma \gg 1$, the chance of index misidentification is low, and thus, it is preferable to allocate more samples to the task of mean estimation.
On the other extreme, as $\abs{\mu_1-\mu_2}/\sigma \rightarrow 0$, the two distributions are effectively identical, 
and the task of index identification becomes irrelevant. Thus, most of the samples should be utilized for mean estimation.
Interestingly, in both these extremities, it is beneficial to allocate more samples to mean estimation over index identification.

As argued above, minimizing the MSE of the estimator can be achieved by controlling the relative number of samples in each estimation phase. We will now formally define the ensemble estimator and prove that it is equivalent to changing the relative allocation of samples between index-estimation and mean-estimation (\cref{prop:proxy}). Finally, we demonstrate both theoretically and numerically that EE achieves lower MSE than DE, and by doing so, better balances the two sources of errors. 

Ensemble estimation is defined as follows:
While the DE divides the samples into two sets, the EE divides the samples of each random variable $a\in[m]$ into $K$ equal-sized disjoint subsets, such that $S_a = \bigcup_{k=1}^K S_a^{(k)}$ and $S_a^{(k)}\! \cap S_a^{(l)} = \emptyset$, $\forall k\neq l \in \sbrac{K}$. 
We further denote the empirical mean of the $a^{th}$ component of $X$, based on the $k^{th}$ set, by $\hat{\mu}_a^{(k)}$. 
As in double-estimation, EE uses a two-phase procedure to estimate the maximal mean.
First, a single arbitrary set $\tilde{k}\in[K]$ is used to estimate the \textbf{index} of the maximal expectation, 
$\hat{a}^* = \argmax_a \hat{\mu}_a^{(\tilde{k})}$. 
Then, EE utilizes the remaining ensemble members to jointly estimate
$\mu_{\hat{a}^*}$, namely,
\begin{align*}
    \hat{\mu}_{\EE}^* &\triangleq
    \frac{1}{K-1}\sum_{j \in [K] \backslash \tilde{k}}\hat{\mu}_{\hat{a}^*}^{(j)} \\
    &\approx \max_a \dsE \sbrac{\frac{1}{K-1}\sum_{j \in [K] \backslash \tilde{k}} \hat{\mu}_a^{(j)}}
    =\mu^*~.
\end{align*}
\cref{lemma:EE underestimation} shows that by using EE, we still maintain the underestimation nature of DE.
\begin{restatable}{lemma}{SEbias}
\label{lemma:EE underestimation}
Let $\M =\argmax_a \dsE \sbrac{{X_a}}$ 
be the set of indices where $\dsE[X_a]$ is maximized. 
Let $\hat{a}^*\in\argmax_a \hat{\mu}^{\brac{\tilde{k}}}_a$ be the estimated maximal-expectation index according to the $\tilde{k}^{th}$ subset. 
Then 
$\dsE \sbrac{\hat{\mu}_{\EE}^*}  = \dsE \sbrac{\mu_{\hat{a}^*}} \leq \max_a \dsE \sbrac{X_a}$. 
Moreover, the inequality is strict if and only if $P\brac{\hat{a}^* \notin \M} > 0$.
\end{restatable}
The proof can be found in \cref{proof:lemma1}. 
In addition, below, we show that changing the size of the ensemble is equivalent to controlling the partition of data between the two estimation phases. As previously mentioned, the split-ratio greatly affects the MSE, and therefore, this property enables the analysis of the effect of the ensemble size $K$ on the MSE.

\begin{restatable}[The Proxy Identity]{proposition}{proxyIdentity}\label{prop:proxy}
 Let $S^{(1)},\dots,S^{(K)}$ be some subsets of the samples of equal sizes $N/K$ and let $\tilde{k}\in\sbrac{K}$ be an arbitrary index. Also, denote by $\hat{\mu}_{\EE}^*$, the EE that uses the $\tilde{k}^{th}$ subset for its index-estimation phase. 
Finally, let $\hat{\mu}_{\WDE}^*$ be a DE that uses $S^{(\tilde{k})}$ for its index-selection and all other samples for mean-estimation; i.e., $\hat{a}^* = \argmax_a \hat{\mu}^{(\tilde{k})}_a$ and $\hat{\mu}_{\WDE}^*=\hat{\mu}_{\hat{a}^*}\brac{\bigcup_{j \in [K]\backslash \tilde{k}} S^{(j)}}$. Then, $\hat{\mu}_{\EE}^* = \hat{\mu}_{\WDE}^*$.
\end{restatable}

We call the weighted version of the DE that uses $1/K$ of its samples to the index estimation W-DE. The proxy identity establishes that W-DE is equivalent to the ensemble estimator. 
Specifically, let $N_1$ be the number of samples used in the first phase of W-DE (index estimation) and say that the MSE of the W-DE is minimized at $N_1 = N_1^*$. Then, by the proxy-identity, the optimal ensemble size of the EE is $K \approx N/N_1^*$. Thus, the optimal split-ratio of W-DE serves as a \textit{proxy} to the number of estimators to be used in EE. 


To better understand the optimal split-ratio, we analyze the MSE of the EE. To this end, we utilize the proxy identity and instead calculate the MSE of the W-DE with $N_1\approx N/K$ samples for index-estimation. Let $X$ be a vector with independent component of means $\brac{\mu_1, \dots, \mu_m}$ and variances $\brac{\sigma_1^2, \dots, \sigma_m^2}$. Assume w.l.o.g.\ that the means are sorted such that $\mu_1 \ge \mu_2 \ge \dots \ge \mu_m$. Then, the following holds (see derivations in Appendix \ref{Ensemble estimator MSE}):
    \begin{align*}
    &\bias(\hat{\mu}^*_{\WDE})=
    \sum_{a=1}^m \brac{\mu_a-\mu_1} P(\hat{a}^*=a), \\
    &\var(\hat{\mu}^*_{\WDE}) 
    =\sum_{a=1}^m \brac{ \frac{\sigma_a^2}{N-N_1}+ \mu_a^2 } P(\hat{a}^*=a) \\
    &\qquad\qquad\qquad- \brac{\sum_{a=1}^m \mu_a P(\hat{a}^*=a)}^2.
    \end{align*}
    
    As $\mu_1$ is assumed to be largest, the bias is always negative; hence, EE underestimates the maximal mean.
    Furthermore, we derive a closed form for the MSE:
    \begin{align} \label{eq:mse}
    \MSE&\brac{\hat{\mu}^*_{\WDE}}  \nonumber\\
    &= \sum_{a=1}^m\brac{\frac{\sigma_a^2}{N-N_1} + (\mu_1-\mu_a)^2}P\brac{\hat{a}^*=a}.
    \end{align}
We hypothesize that in most cases, the MSE is minimized for $K>2$.
As \eqref{eq:mse} is rather cumbersome to analyze, we prove this hypothesis in a simpler case and demonstrate it numerically for more challenging situations.
\begin{restatable}{proposition}{propSplitRatio}\label{prop:split_ratio}
    Let 
    $X=(X_1,X_2)\sim \N\brac{(\mu_1,\mu_2)^T,\sigma^2I_2}$ be a Gaussian random vector
    such that $\mu_1\ge\mu_2$ 
    and let $\Delta=\mu_1-\mu_2$. Also, define the signal to noise ratio as $\SNR=\frac{\Delta}{\sigma/\sqrt{N}}$ and let $\hat{\mu}^*_{\WDE}$ be a W-DE that uses $N_1$ samples for index estimation. Then, for any fixed even sample-size $N>10$ and any $N_1^*$ that minimizes $\MSE(\hat{\mu}^*_{\WDE})$, it holds that
    \begin{enumerate*}
        \item[(1)] \label{large_snr}As $\SNR \to \infty$, $N_1^*\to 1$
        \item[(2)] \label{small_snr}As $\SNR \to 0$, $N_1^*\to 1$ 
        \item[(3)] \label{any_case_EE_better}For any $\sigma$ and $\Delta$, it holds that $N_1^*<N/2$.
    \end{enumerate*}
\end{restatable}
The proof can be found in \cref{proof:prop2}. Note that a similar analysis can be done for sub-Gaussian variables, using standard concentration bounds instead of using $\Phi$. However, in this case, we have to bound $P(\hat{a}^*=a)$ and can only analyze an upper bound of the MSE.

\begin{figure}[t]
\begin{center}
\centerline{\includegraphics[width=0.9\columnwidth]{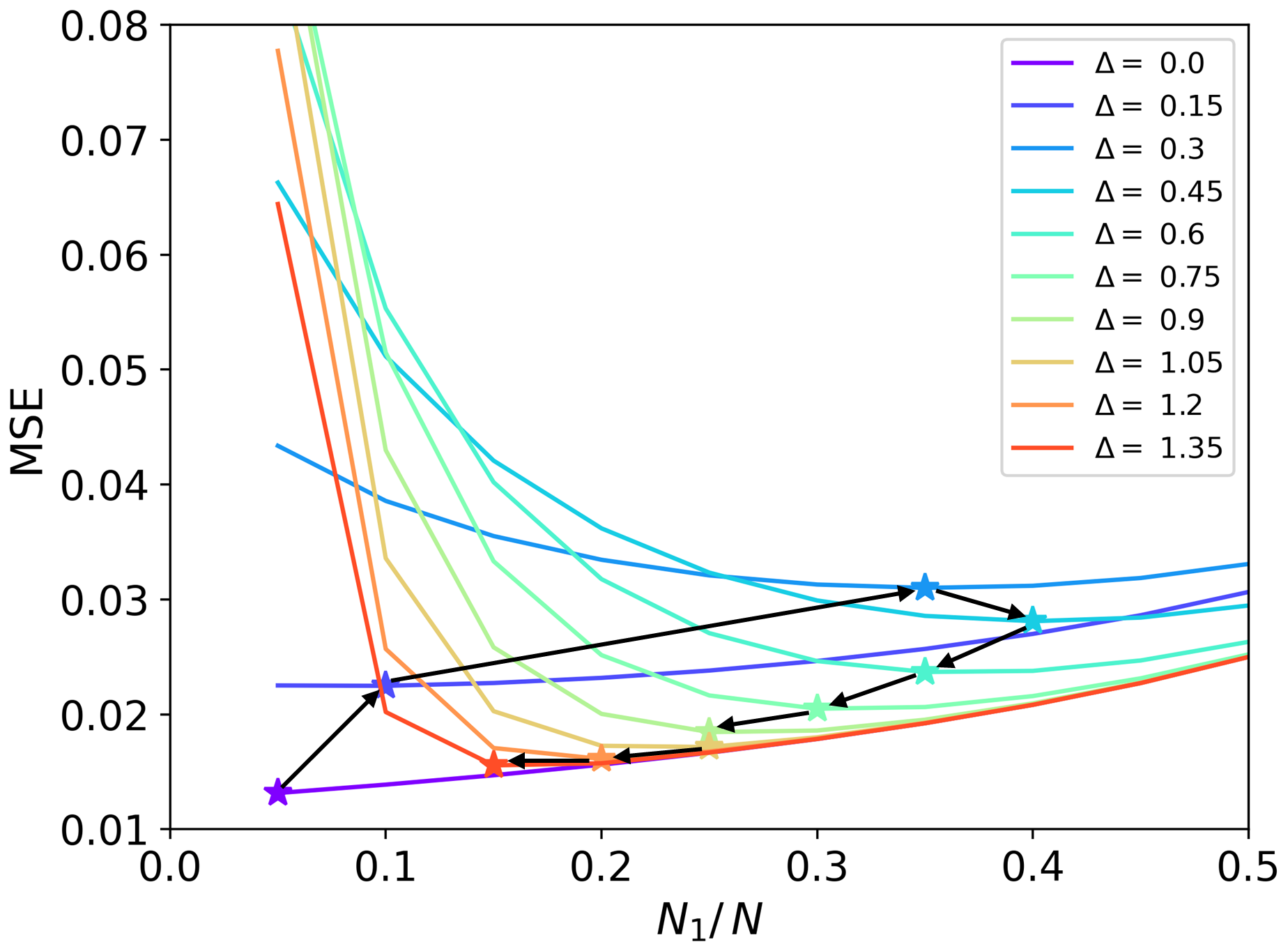}}
\caption{The MSE values as function of the split-ratio for two Gaussians with different mean-gaps $\Delta=\mu_1-\mu_2$ and $\sigma^2=0.25$. Stars mark values in which the minimum MSE is achieved. As expected from \cref{prop:split_ratio}, the optimal split-ratio is always smaller than $N/2$. The black arrows show the trend of the optimum as $\Delta$ increases. Notably, the optimal split-ratio increases from $0$ to $0.4$ and then decreases back to $0$, with alignment to the $\SNR$ claims.
}
\label{fig:MSE as func of r}
\end{center}
\end{figure}

\begin{figure}[t]
\begin{center}
\centerline{\includegraphics[width=0.9\columnwidth]{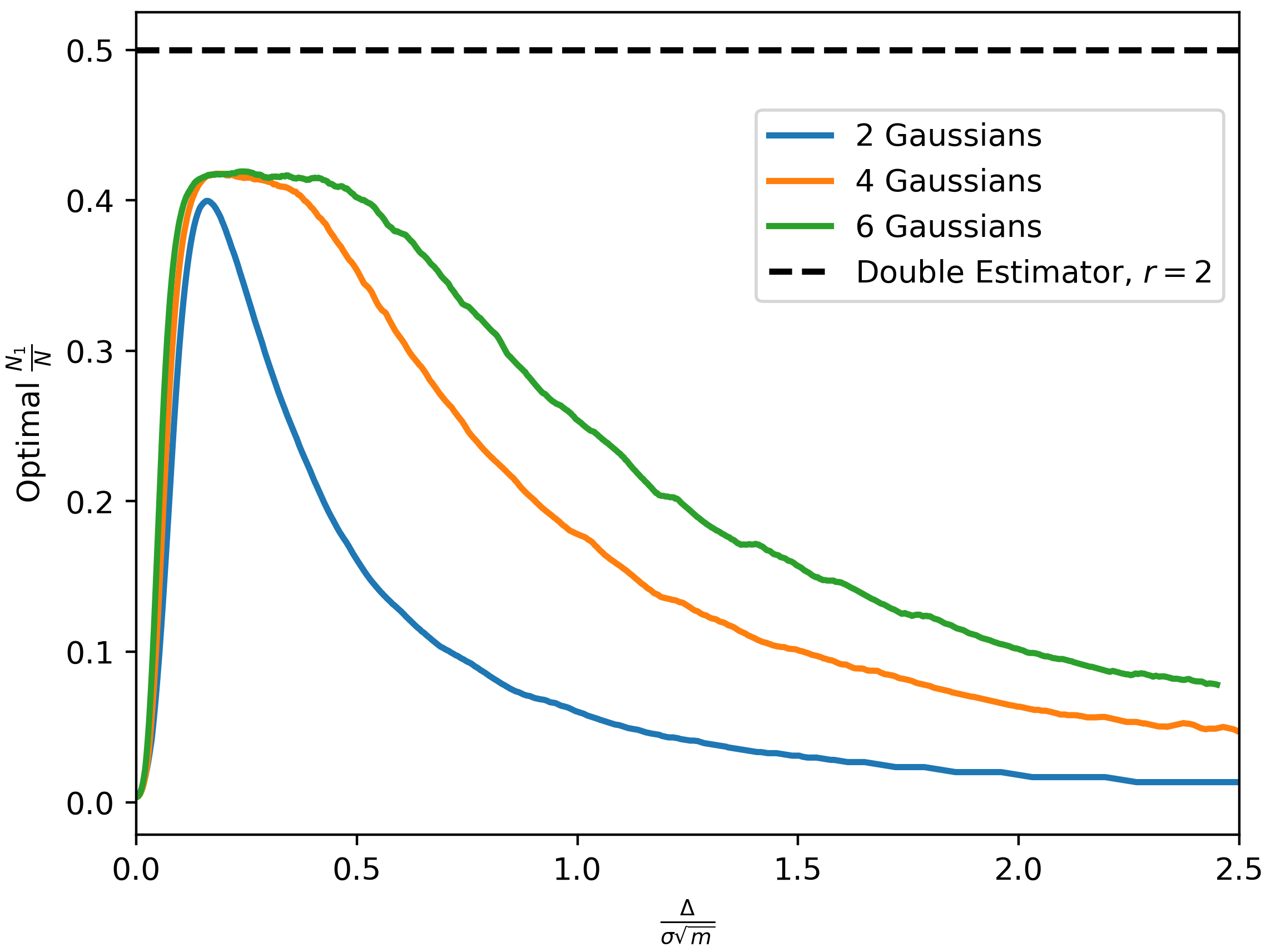}}
\caption{
Numerical calculation of the optimal sample split-ratio as function of normalized gap for 2, 4 and 6 Gaussians with means uniformly spread over the interval $\Delta\!=\!\mu_{max}\!-\!\mu_{min}$.}
\label{fig:optimal r as func of delta}
\end{center}
\end{figure}

\Cref{prop:split_ratio} yields a rather intuitive, yet surprising result. From the first claim, we learn that when $X_1$ and $X_2$ are easily distinguishable (high $\SNR$), samples should not be invested in index-estimation, but rather allocated to the mean-estimation. Moreover, the second claim establishes the same conclusion for the case where $X_1$ and $X_2$ are indiscernible (low $\SNR$). Then, $X_1$ and $X_2$ have near-identical means, and samples are better used for reducing the variance of any of them. The most remarkable claim is the last one; it implies that that the optimal split ratio is \textit{always} smaller than half. In turn, when $N$ is large enough, this implies that there exists $K>2$ such that the MSE of EE is strictly lower than the one of DE.
To further demonstrate our claim for intermediate values of the $\SNR$, \cref{fig:MSE as func of r} shows the MSE as a function of $N_1$ for different values of $\Delta=\mu_1-\mu_2$ (and fixed variance $\sigma^2=0.25$). 

%

In \cref{fig:optimal r as func of delta} we plot the value of $N_1$ that minimizes the MSE as a function of the normalized-distance of the means, $\frac{\Delta}{\sigma\sqrt{m}}$, for a different number of random variables $m \in \cbrac{2,4,6}$. 
The means are evenly spread on the interval $\Delta$, all with $\sigma^2\!=\!0.25$.
\cref{fig:optimal r as func of delta} further demonstrates that using $N_1<\frac{N}{2}$ can reduce the MSE in a large range of scenarios. By the proxy-identity, this is equivalent to using ensemble sizes of $K>2$. 
Therefore, we expect that in practice, EE will outperform DE in terms of the minimal MSE (MMSE). 

\section{\algfull}

In \cref{subsec: ensemble estimator}, we presented the ensemble estimator and its benefits. Namely, ensembles allow unevenly splitting samples between those used to select the index of the maximal component and those used to approximate its mean. 
We also showed that allocating more samples to estimate the mean, rather than the index, reduces the MSE in various scenarios.

\begin{figure*}[!ht]
    \centering
    \begin{subfigure}[b]{0.31\textwidth}
        \includegraphics[width=\textwidth]{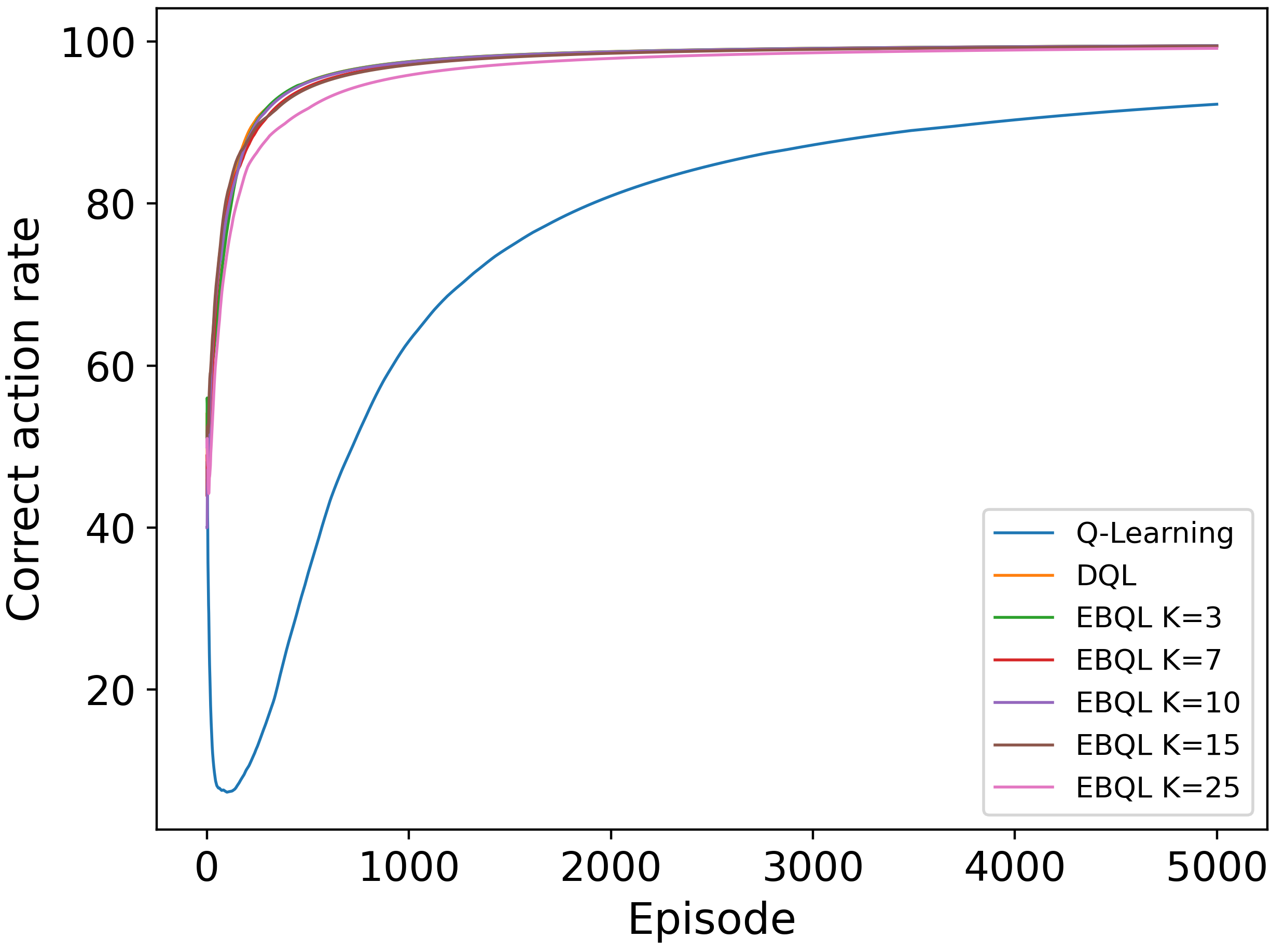}
        \caption{}
        \label{subfig:neg_mu}
    \end{subfigure}%
    ~ 
       \quad
    \begin{subfigure}[b]{0.31\textwidth}
        \includegraphics[width=\textwidth]{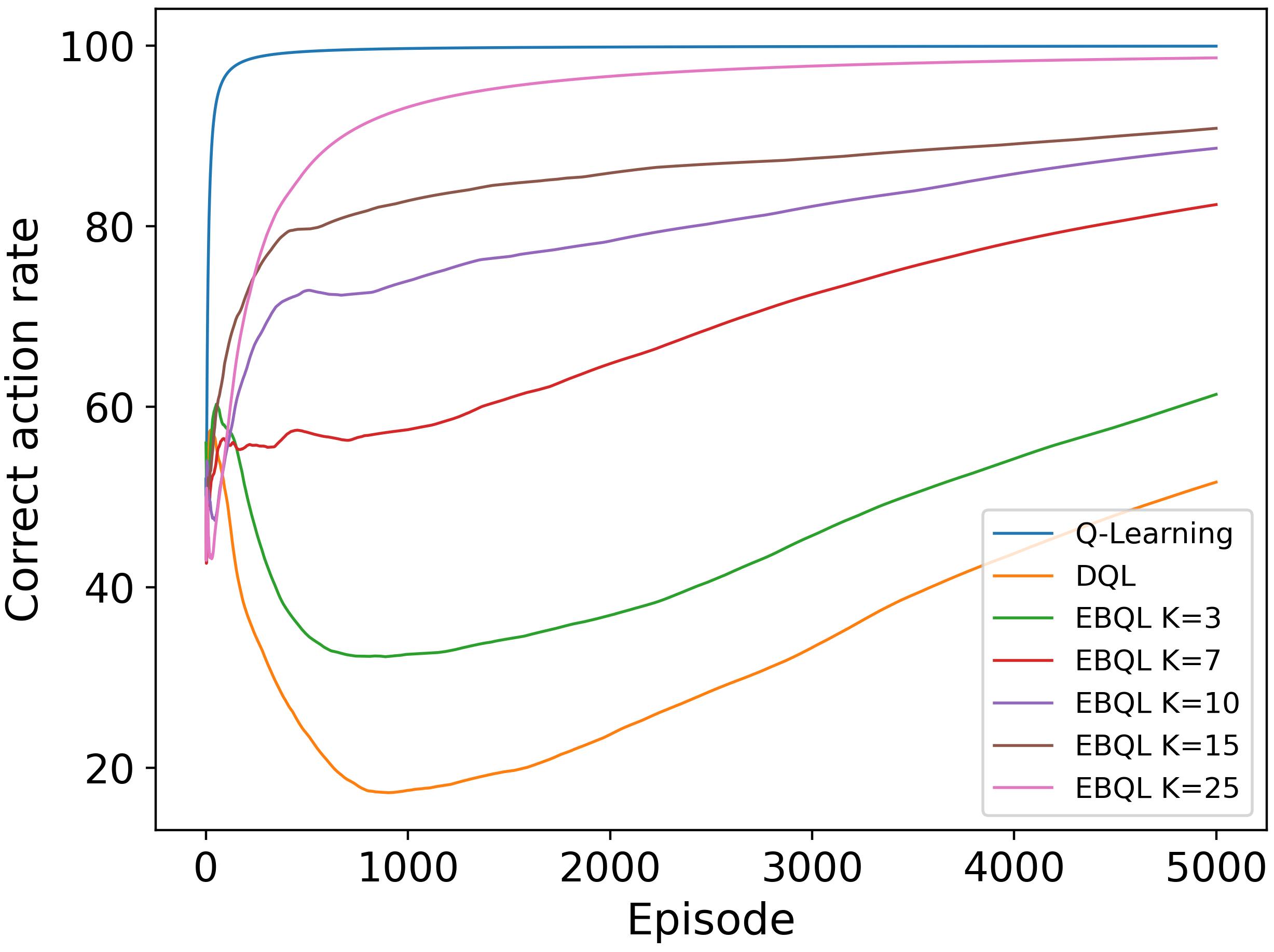}
        \caption{}
        \label{subfig:pos_mu}
    \end{subfigure}
    ~ 
      \quad
    \begin{subfigure}[b]{0.31\textwidth}
        \includegraphics[width=\textwidth]{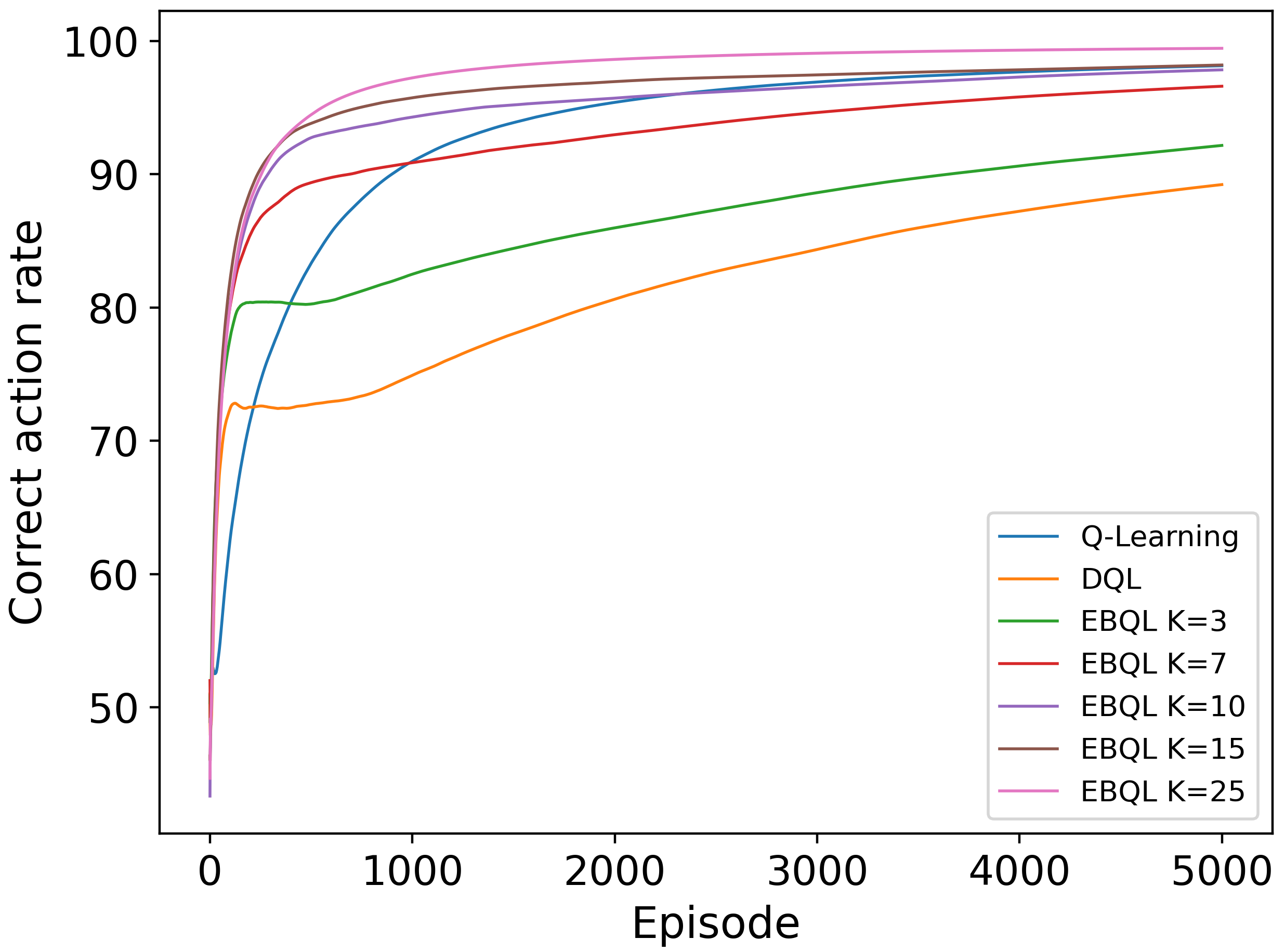}
        \caption{}
        \label{subfig:avg_mu}
    \end{subfigure}
    \caption{Correct-action rate from state $A^i$ as function of episode in the Meta Chain MDP. 
    \cref{subfig:neg_mu} shows the results of a specific chain-MDP where $\mu_i=-0.2<0$, where in \cref{subfig:pos_mu}, $\mu_i=0.2>0$.
    In \cref{subfig:avg_mu}, the average of 6 symmetric $\mu$-values is presented. All results are averaged over 50 random seeds.}
    \label{fig:results_MCMDP}
\end{figure*}

Although the analysis in \cref{subsec: ensemble estimator} focuses on a statistical framework, a similar operation is performed by the QL agent during training when it bootstraps on next-state values. Specifically, recall that $R^{\pi}(s, a)$ is the reward-to-go when starting from state $s$ and action $a$ (see \cref{eq:reward-to-go}). Then, 
denoting $\brac{X_1, \dots, X_m} = \brac{R^{\pi}(s_{t+1}, a_1), \dots, R^{\pi}(s_{t+1}, a_m)}$ and $\brac{\mu_1, \dots, \mu_m} = \brac{Q^{\pi}(s_{t+1}, a_1), \dots, Q^{\pi}(s_{t+1}, a_m)}$, the next-state value used by QL is determined by $\max_a \mu_a$.
As the real $Q$-values are unknown, DQL uses two $Q$ estimators $\brac{\mu_1^{(1)}, \dots, \mu_m^{(1)}} = Q^A$ and $\brac{\mu_1^{(2)}, \dots, \mu_m^{(2)}} = Q^B$ in a way that resembles the DE -- the index is selected as $a^*=\argmax_a \mu_a^{(1)}$ and the value as $\mu^{(2)}_{a^*}$.


 

Our method, \algfull{} (\alg) presented in \cref{alg:EBQL}, can be seen as applying the concepts of EE to the QL bootstrapping phase by using an ensemble of $K$ Q-function estimators. Similar to DQL, when updating the $k^{th}$ ensemble member in \alg{}, we define the next-state action as $\hat{a}^* = \argmax_a Q^k (s_{t+1}, a)$. However, while DQL utilizes two estimators, in \alg{}, the value is obtained by averaging over the remaining ensemble members $Q^{EN \backslash k} = \frac{1}{K-1} \sum_{j \in [K]\backslash k} Q^j(s_{t+1}, \hat{a}^*)$. Notice how when $K=2$ (two ensemble members) we recover the same update scheme as DQL, showing how \alg{} is a natural extension of DQL to ensembles. In practice, the update is done using a learning rate $\alpha_t$, as in DQL (\cref{eq:DQL update}).

\renewcommand{\algorithmiccomment}[1]{#1}
\begin{algorithm}[t]
  \caption{\algfull{} (\alg)}
  \label{alg:EBQL}
\begin{algorithmic}
  \STATE {\bfseries Parameters:} learning-rates: $\cbrac{\alpha_t}_{t\ge1}$
  \STATE {\bfseries Initialize:} Q-ensemble of size $K: \cbrac{Q^i}_{i=1}^K$ , $s_0$
  \FOR{$t= 0, \dots , T$}
    \STATE Choose action $a_t=\argmax_{a}
    \sbrac{\sum_{i=1}^K Q^i(s_t,a)}$ 
    \STATE $a_t = \text{explore}(a_t)$  \hfill\COMMENT{//e.g. $\epsilon$-greedy}
    \STATE $s_{t+1}, r_t \gets \text{env.step}(s_t,a_t)$ 
    \STATE Sample an ensemble member to update: $k_t \!\sim\! \U\!\brac{\sbrac{K}}$
    \STATE $\hat{a}^* = \argmax_a Q^{k_t}(s_{t+1},a)$
    \STATE $Q^{k_t}(s_t,a_t) \gets (1-\alpha_t)Q^{k_t}(s_t,a_t)$\\  $\qquad\qquad\qquad + \alpha_t\brac{r_t + \gamma Q^{EN \setminus k_t}(s_{t+1},\hat{a}^*)}$
  \ENDFOR
  \STATE \textbf{Return} $\cbrac{Q^i}_{i=1}^K$
\end{algorithmic}
\end{algorithm}

\begin{figure}[t]
\begin{center}
\centerline{\includegraphics[width=.975\columnwidth]{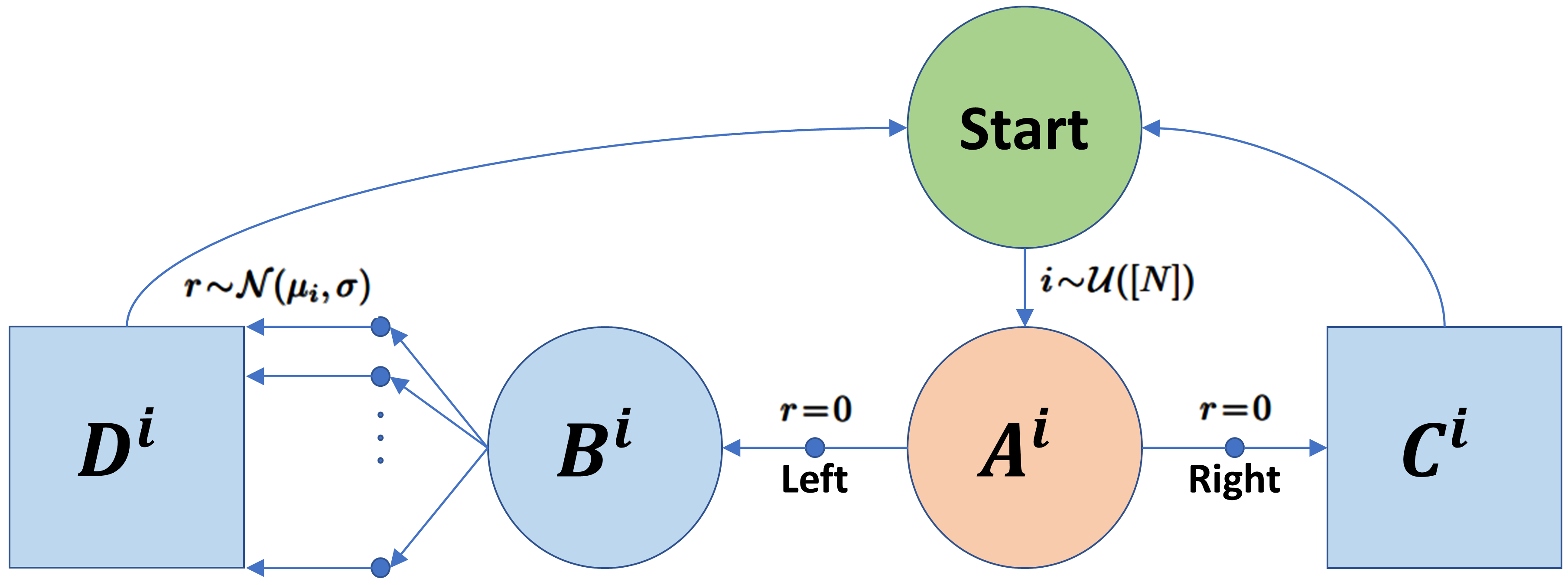}}
\caption{The Meta-Chain-MDP, a set of chain-MDPs differing by the reward at state $D$. Each episode the agent is randomly initialized in one of the MDPs. The variance $\sigma=1$ is identical across the chain MDPs.}
\label{fig:MCMDP}
\end{center}
\end{figure}

\begin{figure}[t]
\begin{center}
\centerline{\includegraphics[width=0.9\columnwidth]{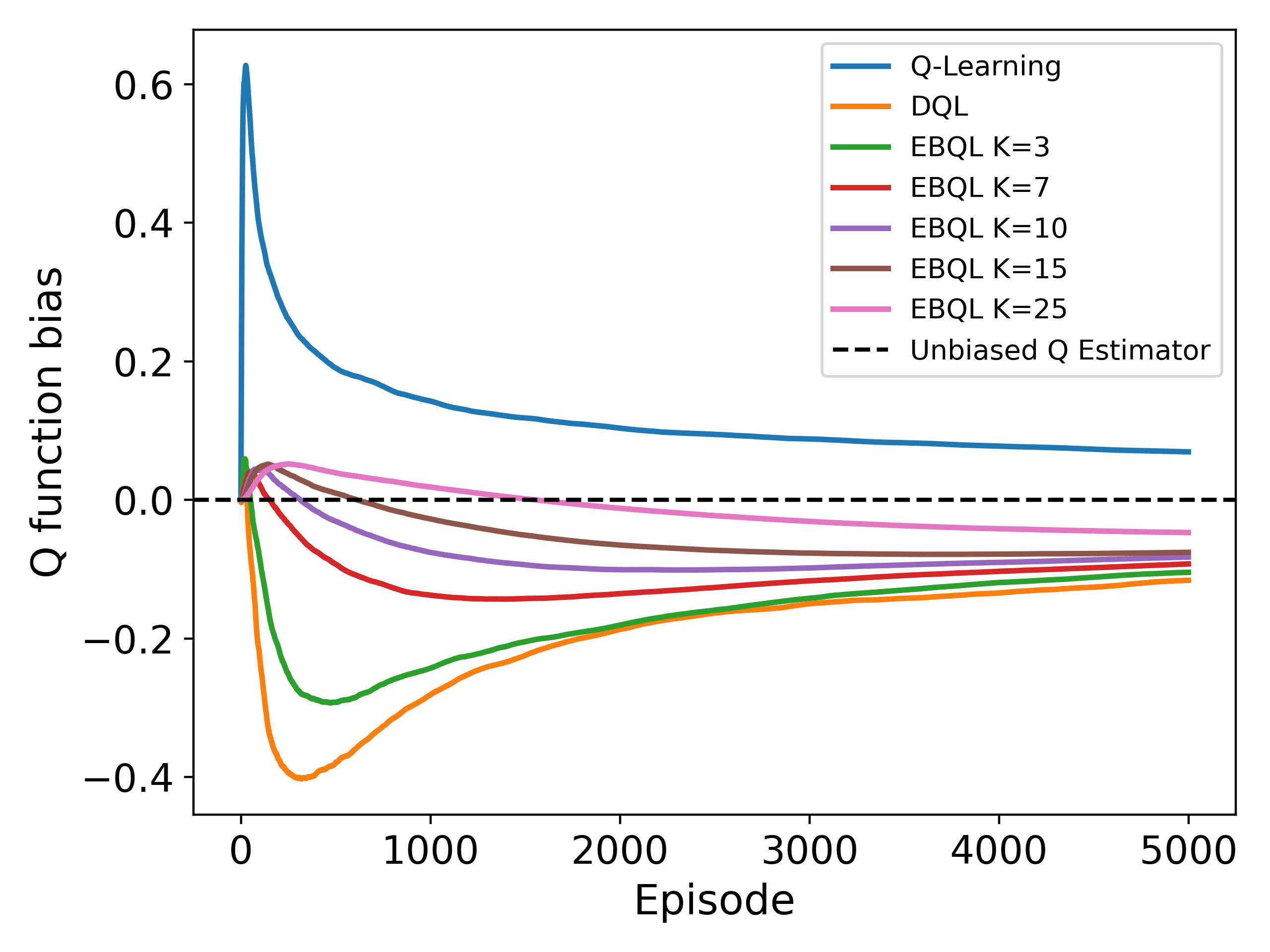}}
\caption{
Estimation bias of the optimal action in state $A^i$ as a function of time in the meta-chain MDP. Due to the zero-initialization of the Q-tables, both DQL and \alg{} are positively biased for a short period of time.
}
\label{fig:Estimation bias}
\end{center}
\end{figure}

\section{Experiments} \label{Experimints}
In this section, we present two main experimental results of \alg{} compared to QL and DQL in both a tabular setting and on the ATARI ALE \citep{bellemare2013arcade} using the deep RL variants. We provide additional details on the methodology in \cref{appendix:implementation}.

\subsection{Tabular Experiment - Meta Chain MDP}\label{exp: tabular}

We begin by analyzing the empirical behavior of the various algorithms on a \emph{Meta Chain MDP} problem (\cref{fig:MCMDP}). 
In this setting, an agent repeatedly interacts with a predetermined set of $N$ chain-MDPs. All chain-MDPs share the same structure, and by the end of the interaction with one MDP, the next MDP is sampled uniformly at random.


For each chain-MDP $i\in [N]$, the agent starts at state $A^i$ and can move either left or right. 
While both immediate rewards are 0, transitioning to state $C^i$ terminates the episode. 
However, from state $B^i$, the agent can perform one of $m$ actions, all transitioning to state $D^i$ and providing the same reward $r (D^i) \sim \mathcal{N}(\mu_i, \sigma^2)$, and the episode is terminated. 
The set of chain MDPs is a symmetrical mixture of chains with positive and negative means. While a negative $\mu_i$ amplifies the sub-optimality of QL due to the over-estimation bootstrapping bias, we observe that positive $\mu_i$ presents the sub-optimality of DQL under-estimation bias.


This simple, yet challenging MDP serves as a playground to examine whether an algorithm balances optimism and pessimism. It also provides motivation for the general case, where we do not know if optimism or pessimism is better.


\textbf{Analysis:} We present the empirical results in \cref{fig:results_MCMDP}. In addition to the results on the meta chain MDP, we also present results on the single-chain setting, for both positive and negative reward means $\mu_i$.

As expected, due to its optimistic nature, QL excels in sub-MDPs where $\mu_i > 0$ (\cref{subfig:pos_mu}). 
The optimistic nature of QL drives the agent towards regions with higher uncertainty (variance). 
As argued above, we observe that in this case, DQL is sub-optimal due to the pessimistic nature of the estimator.
On the other hand, due to its pessimistic nature, and as shown in \citet{hasselt2010double}, DQL excels when $\mu_i < 0$. We replicate their results in \cref{subfig:neg_mu}.

 
Although QL and DQL are capable of exploiting the nature of certain MDPs, this may result in catastrophic failures. On the other hand, and as shown in \cref{subfig:avg_mu}, \alg{} excels in the meta chain MDP, a scenario that averages over the sub-MDPs, showing the robustness of the method to a more general MDP. 
Moreover, we observe that as the size of the ensemble grows, the performance of \alg{} improves.

\textbf{Estimation Bias:} In addition to reporting the performance, we empirically compare the bias of the various learned Q-functions on the meta chain MDP. The results are presented in \cref{fig:Estimation bias}. As expected QL and DQL converge to $0^+$ and $0^-$ respectively. 
Interestingly, the absolute bias of \alg{} decreases and the algorithm becomes less pessimistic as the ensemble size $K$ grows. This explains the results of \cref{fig:results_MCMDP}, where increasing the ensemble size greatly improves the performance in chain MDPs with $\mu_i>0$. To maintain fairness, all Q-tables were set to zero, which can explain the positive bias of \alg{} in the initial episodes.

\begin{figure}[t]
\begin{center}
\centerline{\includegraphics[width=0.9\columnwidth]{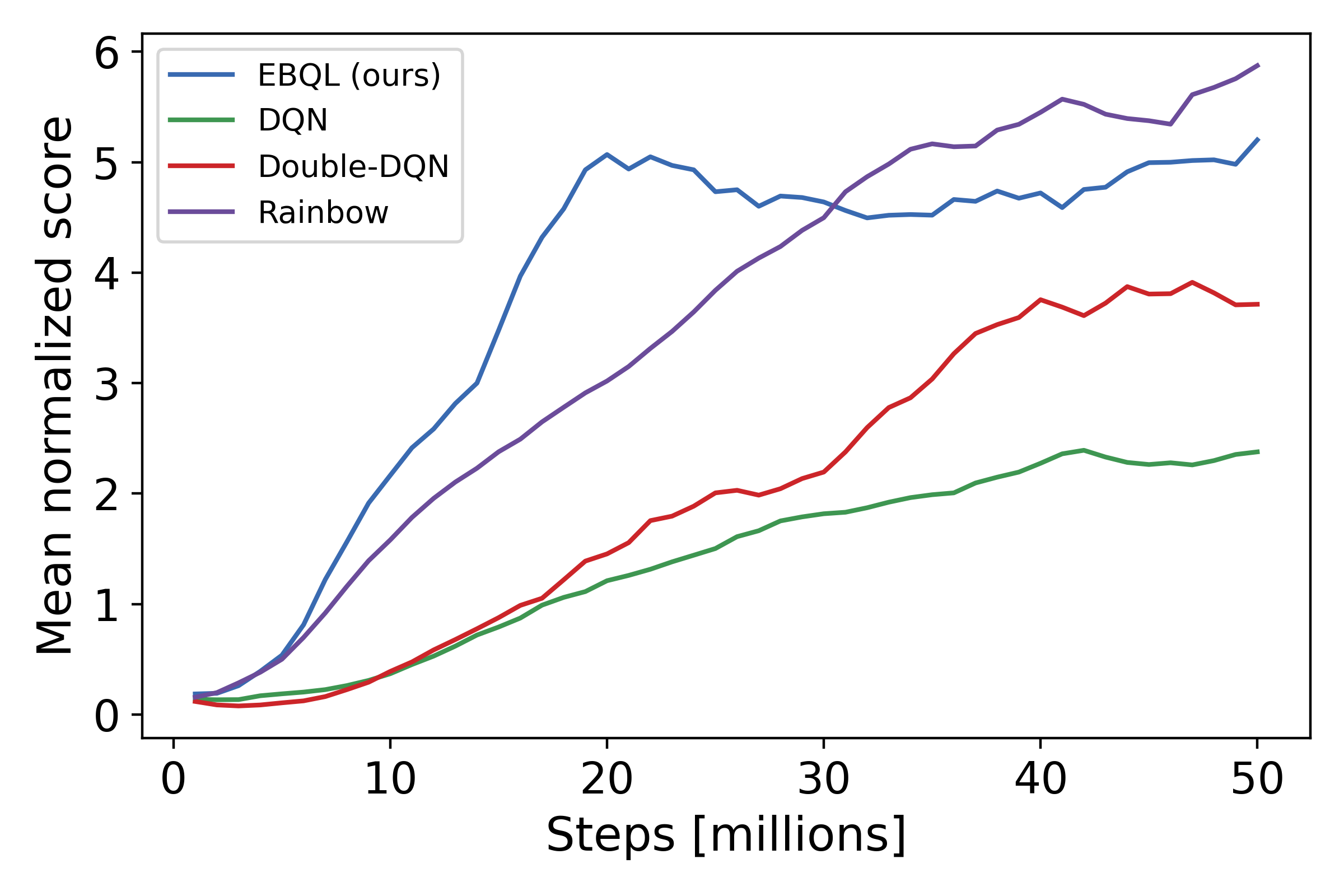}}
\centerline{\includegraphics[width=0.9\columnwidth]{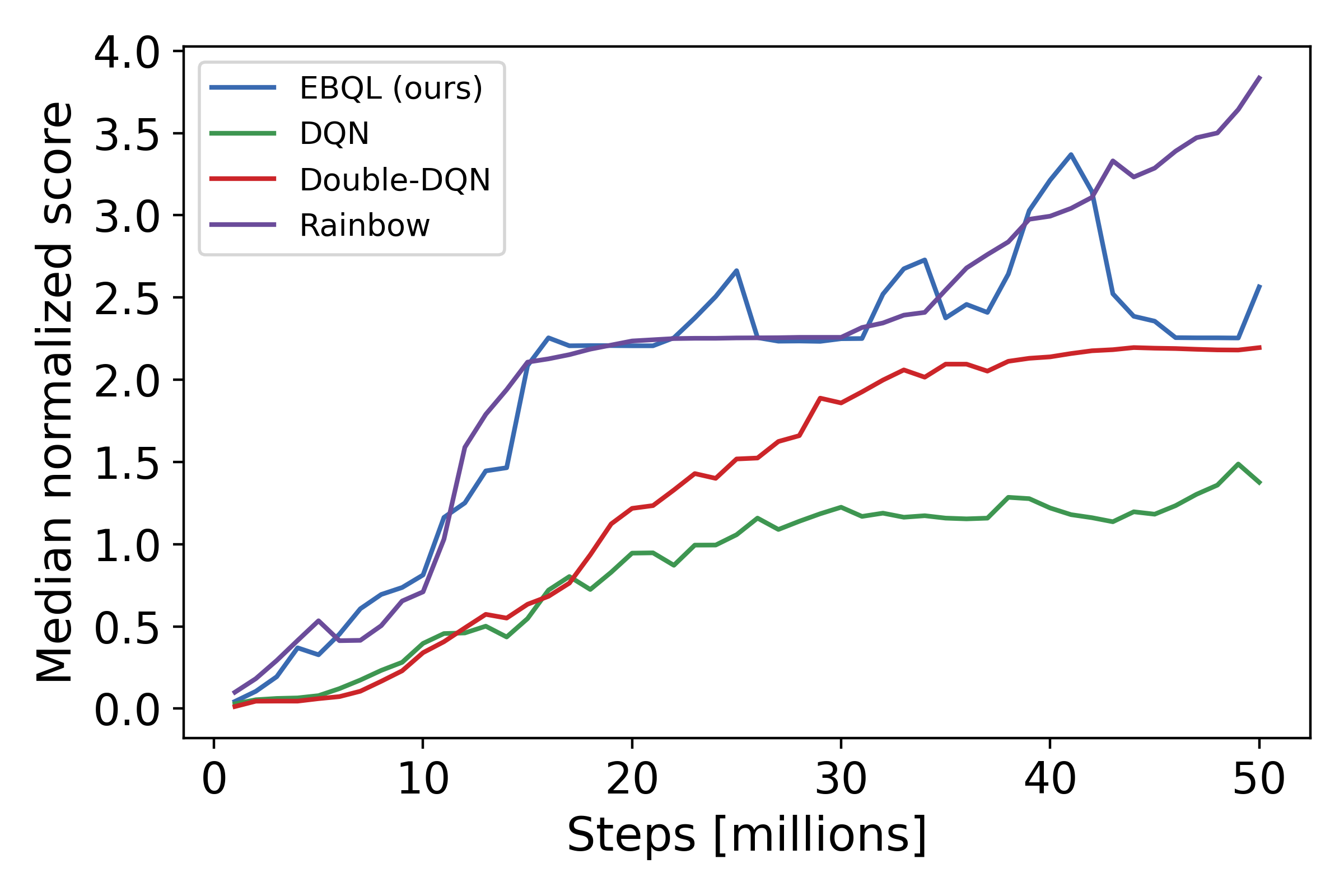}}
\caption{\textbf{Normalized scores:} We present both the mean and median normalized scores over the 11 environments (\cref{tab: atari 50m steps}) with 5 seeds per environment. Curves are smoothed with a moving average over 5 points. In addition to DQN and DDQN, we provide a comparison against Rainbow \citep{hessel2018rainbow}. Rainbow represents the state-of-the-art non-distributed Q-learning-based agent.
}
\label{fig: normalized scores}
\end{center}
\end{figure}

\begin{table*}[t]
    \centering
    \caption{\textbf{ATARI:} Comparison of the DQN, DDQN and \alg{} agents on 11 random ATARI environments. Each algorithm is evaluated for 50m steps over 5 random seeds. We present the average score of the final policies.}
    \label{tab: atari 50m steps}
    \begin{tabular}{c||c|c||c|c||c}
        \backslashbox{\textbf{Environment}}{\textbf{Algorithm}} & \textbf{Human} & \textbf{Random} & \textbf{DQN} & \textbf{DDQN} & \textbf{\alg{}} \\
        \hline\hline
        \textbf{Asterix} & 8503.3 & 210.0 & 4603.2 & 5718.1 & \textbf{22152.5} \\
        \hline
        \textbf{Breakout} & 31.8 & 1.7 & 283.5 & 333.4 & \textbf{406.3} \\
        \hline
        \textbf{CrazyClimber} & 35410.5 & 10780.5 & 93677.1 & 111765.4 & \textbf{127967.5} \\
        \hline
        \textbf{DoubleDunk} & -15.5 & -18.6 & -15.8 & -18.6 & \textbf{-10.1} \\
        \hline
        \textbf{Gopher} & 2321.0 & 257.6 & 3439.4 & 6940.5 & \textbf{21940.0} \\
        \hline
        \textbf{Pong} & 9.3 & -20.7 & 17.4 & 19.9 & \textbf{21.0} \\
        \hline
        \textbf{PrivateEye} & \textbf{69571.3} & 24.9 & 78.1 & 100.0 & 100.0 \\
        \hline
        \textbf{Qbert} & 13455.0 & 163.9 & 5280.2 & 6277.1 & \textbf{14384.4} \\
        \hline
        \textbf{RoadRunner} & 7845.0 & 11.5 & 27671.4 & 40264.9 & \textbf{55927.5} \\
        \hline
        \textbf{Tennis} & -8.9 & -23.8 & -1.2 & -14.5 & \textbf{-1.1} \\
        \hline
        \textbf{VideoPinball} & 17297.6 & 16256.9 & 104146.4 & 230534.3 & \textbf{361205.1}
    \end{tabular}
\end{table*}

\subsection{Atari}\label{experiments: atari}

Here, we evaluate \alg{} in a high dimensional task -- ATARI ALE \citep{bellemare2013arcade}. We present \alg's performance when $K=5$ (ensemble size).
Here, we compare to the DQN \citep{mnih2015human} and DDQN \citep{van2016deep} algorithms, the deep RL variants of Q-learning and Double-Q-Learning respectively.

While there exists a multitude of extensions and practical improvements to DQN, our work focuses on learning the Q-function at a fundamental level. 
Similarly to what is shown in Rainbow \citep{hessel2018rainbow}, most of these methods are orthogonal and will thus only complement our approach.

In order to obtain a fair comparison, we \emph{do not} use the ensemble for improved exploration (as was done in \citealt{osband2016deep,osband2018randomized,lee2020sunrise}). Rather, the action is selected using a standard $\epsilon$-greedy exploration scheme, where the greedy estimator is taken on the average Q-values of the ensemble. All hyper-parameters are identical to the baselines, as reported in \citep{mnih2015human}, including the use of target networks (see \cref{appendix:implementation} for pseudo-codes).

\textbf{Analysis:} We test all methods on 11 randomly-chosen environments. Each algorithm is trained over 50m steps and 5 random seeds. To ensure fairness, for DQN and DDQN we report the results as provided by \citet{dqnzoo2020github}. The complete training curves over all environments are provided in \cref{appendix:experimental results}.
The numerical results are shown in \cref{tab: atari 50m steps}. In addition to the numerical results, we present the mean and median normalized scores in \cref{fig: normalized scores}.

When comparing our method to DQN and DDQN, we observe that, in the environments we tested, \alg{} consistently obtains the \textit{highest} performance. In addition, PrivateEye is the only evaluated environment where \alg{} did not outperform the human baseline. Remarkably, despite \alg{} not containing many improvements (e.g., prioritized experience replay \citep{schaul2015prioritized}, distributional estimators \citep{bellemare2017distributional}, and more \citep{hessel2018rainbow}), it outperforms the more complex Rainbow in the small-sample region ($< 20$ million steps, \cref{fig: normalized scores}). 
Particularly in the Breakout and RoadRunner domains (\cref{appendix:experimental results}), \alg{} also outperforms Rainbow over all 50m training steps.

\section{Summary and Future Work} \label{Summary}

In this work, we presented \algfull, an under-estimating bias-reduction method. We started by analyzing the simple statistical task of estimating the maximal mean of a set of independent random variables. Here, we proved that the ensemble estimator is equivalent to a weighted double-estimator that unevenly splits samples between its two phases (index-selection and mean-estimation). Moreover, we showed both theoretically and empirically that in some scenarios, the optimal partition is not symmetrical. Then, using the ensemble estimator is beneficial.

Based on this analysis, we propose \alg, a way of utilizing the ensemble estimator in RL. Our empirical results in the meta chain MDP show that, compared to QL (single estimator) and DQL (double estimator), \alg{} dramatically reduces the estimation bias while obtaining superior performance. Finally, we evaluated \alg{} on 11 randomly selected ATARI environments. To ensure fairness, we compared DQN, DDQN and \alg{} without any additional improvements \citep{hessel2018rainbow}. In all the evaluated environments, \alg{} exhibits the best performance, in addition to outperforming the human baseline in all but a single environment. Finally, \alg{} performs competitively with Rainbow in the small-samples regime, even though it does not utilize many of its well-known improvements.

Ensembles are becoming prevalent in RL, especially due to their ability to provide uncertainty estimates and improved exploration techniques \citep{osband2018randomized,lee2020sunrise}. Although ensembles are integral to our method, to ensure a fair evaluation, we did not utilize their properties for improved exploration. These results are exciting as they suggest that by utilizing these properties, the performance of \alg{} can be dramatically improved.

We believe that our work leaves room to many interesting extensions. First, recall that 
\alg{} fixes the ensemble size at the beginning of the interaction and then utilizes a single estimator to estimate the action index. However, in \cref{subsec: ensemble estimator}, we showed that the optimal split ratio depends on the distribution of the random variables in question. 
Therefore, one possible extension is to dynamically change the number of ensemble members used for the index-estimation, according to observed properties of the MDP.

Also, when considering the deep RL implementation, we provided a clean comparison of the basic algorithms (QL, DQL and \alg). Notably, and in contrast to previous work on ensembles in RL, we did not use ensembles for improved exploration. However, doing so is non-trivial; for exploration needs, previous work decouples the ensemble members during training, whereas our update does the opposite.



\bibliography{bib.bib}

\begin{thebibliography}{37}
\providecommand{\natexlab}[1]{#1}
\providecommand{\url}[1]{\texttt{#1}}
\expandafter\ifx\csname urlstyle\endcsname\relax
  \providecommand{\doi}[1]{doi: #1}\else
  \providecommand{\doi}{doi: \begingroup \urlstyle{rm}\Url}\fi

\bibitem[Abdulhai et~al.(2003)Abdulhai, Pringle, and
  Karakoulas]{abdulhai2003traffic1}
Abdulhai, B., Pringle, R., and Karakoulas, G.~J.
\newblock Reinforcement learning for true adaptive traffic signal control.
\newblock \emph{Journal of Transportation Engineering}, 129\penalty0
  (3):\penalty0 278--285, 2003.

\bibitem[Andrychowicz et~al.(2020)Andrychowicz, Baker, Chociej, Jozefowicz,
  McGrew, Pachocki, Petron, Plappert, Powell, Ray,
  et~al.]{andrychowicz2020learning}
Andrychowicz, O.~M., Baker, B., Chociej, M., Jozefowicz, R., McGrew, B.,
  Pachocki, J., Petron, A., Plappert, M., Powell, G., Ray, A., et~al.
\newblock Learning dexterous in-hand manipulation.
\newblock \emph{The International Journal of Robotics Research}, 39\penalty0
  (1):\penalty0 3--20, 2020.

\bibitem[Anschel et~al.(2017)Anschel, Baram, and Shimkin]{anschel2017averaged}
Anschel, O., Baram, N., and Shimkin, N.
\newblock Averaged-dqn: Variance reduction and stabilization for deep
  reinforcement learning.
\newblock In \emph{International Conference on Machine Learning}, pp.\
  176--185. PMLR, 2017.

\bibitem[Arulkumaran(2019)]{kaixin}
Arulkumaran, K.
\newblock Rainbow dqn.
\newblock \url{https://https://github.com/Kaixhin/Rainbow}, 2019.

\bibitem[Azar et~al.(2011)Azar, Munos, Ghavamzadeh, and Kappen]{azar2011speedy}
Azar, M.~G., Munos, R., Ghavamzadeh, M., and Kappen, H.
\newblock Speedy q-learning.
\newblock In \emph{Advances in Neural Information Processing Systems}, 2011.

\bibitem[Badia et~al.(2020)Badia, Piot, Kapturowski, Sprechmann, Vitvitskyi,
  Guo, and Blundell]{badia2020agent57}
Badia, A.~P., Piot, B., Kapturowski, S., Sprechmann, P., Vitvitskyi, A., Guo,
  Z.~D., and Blundell, C.
\newblock Agent57: Outperforming the atari human benchmark.
\newblock In \emph{International Conference on Machine Learning}, pp.\
  507--517. PMLR, 2020.

\bibitem[Bellemare et~al.(2013)Bellemare, Naddaf, Veness, and
  Bowling]{bellemare2013arcade}
Bellemare, M.~G., Naddaf, Y., Veness, J., and Bowling, M.
\newblock The arcade learning environment: An evaluation platform for general
  agents.
\newblock \emph{Journal of Artificial Intelligence Research}, 47:\penalty0
  253--279, 2013.

\bibitem[Bellemare et~al.(2017)Bellemare, Dabney, and
  Munos]{bellemare2017distributional}
Bellemare, M.~G., Dabney, W., and Munos, R.
\newblock A distributional perspective on reinforcement learning.
\newblock In \emph{International Conference on Machine Learning}, pp.\
  449--458. PMLR, 2017.

\bibitem[Bellemare et~al.(2020)Bellemare, Candido, Castro, Gong, Machado,
  Moitra, Ponda, and Wang]{bellemare2020autonomous}
Bellemare, M.~G., Candido, S., Castro, P.~S., Gong, J., Machado, M.~C., Moitra,
  S., Ponda, S.~S., and Wang, Z.
\newblock Autonomous navigation of stratospheric balloons using reinforcement
  learning.
\newblock \emph{Nature}, 588\penalty0 (7836):\penalty0 77--82, 2020.

\bibitem[Bellman(1957)]{bellman1957markovian}
Bellman, R.
\newblock A markovian decision process.
\newblock \emph{Journal of mathematics and mechanics}, pp.\  679--684, 1957.

\bibitem[Ernst et~al.(2005)Ernst, Geurts, and Wehenkel]{ernst2005tree}
Ernst, D., Geurts, P., and Wehenkel, L.
\newblock Tree-based batch mode reinforcement learning.
\newblock \emph{Journal of Machine Learning Research}, 6:\penalty0 503--556,
  2005.

\bibitem[Hessel et~al.(2018)Hessel, Modayil, Van~Hasselt, Schaul, Ostrovski,
  Dabney, Horgan, Piot, Azar, and Silver]{hessel2018rainbow}
Hessel, M., Modayil, J., Van~Hasselt, H., Schaul, T., Ostrovski, G., Dabney,
  W., Horgan, D., Piot, B., Azar, M., and Silver, D.
\newblock Rainbow: Combining improvements in deep reinforcement learning.
\newblock In \emph{Proceedings of the AAAI Conference on Artificial
  Intelligence}, volume~32, 2018.

\bibitem[Jin et~al.(2018)Jin, Allen-Zhu, Bubeck, and Jordan]{jin2018q}
Jin, C., Allen-Zhu, Z., Bubeck, S., and Jordan, M.~I.
\newblock Is q-learning provably efficient?
\newblock \emph{arXiv preprint arXiv:1807.03765}, 2018.

\bibitem[Kearns \& Singh(1999)Kearns and Singh]{kearns1999finite}
Kearns, M. and Singh, S.
\newblock Finite-sample convergence rates for q-learning and indirect
  algorithms.
\newblock \emph{Advances in neural information processing systems}, pp.\
  996--1002, 1999.

\bibitem[Kober et~al.(2013)Kober, Bagnell, and Peters]{kober2013reinforcement}
Kober, J., Bagnell, J.~A., and Peters, J.
\newblock Reinforcement learning in robotics: A survey.
\newblock \emph{The International Journal of Robotics Research}, 32\penalty0
  (11):\penalty0 1238--1274, 2013.

\bibitem[Lee et~al.(2020)Lee, Laskin, Srinivas, and Abbeel]{lee2020sunrise}
Lee, K., Laskin, M., Srinivas, A., and Abbeel, P.
\newblock Sunrise: A simple unified framework for ensemble learning in deep
  reinforcement learning.
\newblock \emph{arXiv preprint arXiv:2007.04938}, 2020.

\bibitem[Luong et~al.(2019)Luong, Hoang, Gong, Niyato, Wang, Liang, and
  Kim]{luong2019applications2}
Luong, N.~C., Hoang, D.~T., Gong, S., Niyato, D., Wang, P., Liang, Y.-C., and
  Kim, D.~I.
\newblock Applications of deep reinforcement learning in communications and
  networking: A survey.
\newblock \emph{IEEE Communications Surveys \& Tutorials}, 21\penalty0
  (4):\penalty0 3133--3174, 2019.

\bibitem[Mahmud et~al.(2018)Mahmud, Kaiser, Hussain, and
  Vassanelli]{mahmud2018applications}
Mahmud, M., Kaiser, M.~S., Hussain, A., and Vassanelli, S.
\newblock Applications of deep learning and reinforcement learning to
  biological data.
\newblock \emph{IEEE transactions on neural networks and learning systems},
  29\penalty0 (6):\penalty0 2063--2079, 2018.

\bibitem[Mnih et~al.(2015)Mnih, Kavukcuoglu, Silver, Rusu, Veness, Bellemare,
  Graves, Riedmiller, Fidjeland, Ostrovski, et~al.]{mnih2015human}
Mnih, V., Kavukcuoglu, K., Silver, D., Rusu, A.~A., Veness, J., Bellemare,
  M.~G., Graves, A., Riedmiller, M., Fidjeland, A.~K., Ostrovski, G., et~al.
\newblock Human-level control through deep reinforcement learning.
\newblock \emph{nature}, 518\penalty0 (7540):\penalty0 529--533, 2015.

\bibitem[Osband et~al.(2016)Osband, Blundell, Pritzel, and
  Van~Roy]{osband2016deep}
Osband, I., Blundell, C., Pritzel, A., and Van~Roy, B.
\newblock Deep exploration via bootstrapped dqn.
\newblock In \emph{Advances in neural information processing systems}, pp.\
  4026--4034, 2016.

\bibitem[Osband et~al.(2018)Osband, Aslanides, and
  Cassirer]{osband2018randomized}
Osband, I., Aslanides, J., and Cassirer, A.
\newblock Randomized prior functions for deep reinforcement learning.
\newblock In \emph{Advances in Neural Information Processing Systems}, pp.\
  8617--8629, 2018.

\bibitem[Quan \& Ostrovski(2020)Quan and Ostrovski]{dqnzoo2020github}
Quan, J. and Ostrovski, G.
\newblock {DQN} {Zoo}: Reference implementations of {DQN}-based agents, 2020.
\newblock URL \url{http://github.com/deepmind/dqn_zoo}.

\bibitem[Schaul et~al.(2015)Schaul, Quan, Antonoglou, and
  Silver]{schaul2015prioritized}
Schaul, T., Quan, J., Antonoglou, I., and Silver, D.
\newblock Prioritized experience replay.
\newblock \emph{arXiv preprint arXiv:1511.05952}, 2015.

\bibitem[Smith \& Winkler(2006)Smith and Winkler]{smith2006optimizer}
Smith, J.~E. and Winkler, R.~L.
\newblock The optimizer’s curse: Skepticism and postdecision surprise in
  decision analysis.
\newblock \emph{Management Science}, 52\penalty0 (3):\penalty0 311--322, 2006.

\bibitem[Strehl et~al.(2006)Strehl, Li, Wiewiora, Langford, and
  Littman]{strehl2006pac}
Strehl, A.~L., Li, L., Wiewiora, E., Langford, J., and Littman, M.~L.
\newblock Pac model-free reinforcement learning.
\newblock In \emph{Proceedings of the 23rd international conference on Machine
  learning}, pp.\  881--888, 2006.

\bibitem[Sutton \& Barto(2018)Sutton and Barto]{sutton2018reinforcement}
Sutton, R.~S. and Barto, A.~G.
\newblock \emph{Reinforcement learning: An introduction}.
\newblock MIT press, 2018.

\bibitem[Thaler(2012)]{thaler2012winner}
Thaler, R.~H.
\newblock \emph{The winner's curse: Paradoxes and anomalies of economic life}.
\newblock Simon and Schuster, 2012.

\bibitem[Thrun \& Schwartz(1993)Thrun and Schwartz]{thrun1993issues}
Thrun, S. and Schwartz, A.
\newblock Issues in using function approximation for reinforcement learning.
\newblock In \emph{Proceedings of the 1993 Connectionist Models Summer School
  Hillsdale, NJ. Lawrence Erlbaum}, 1993.

\bibitem[Tsitsiklis(1994)]{tsitsiklis1994asynchronous}
Tsitsiklis, J.~N.
\newblock Asynchronous stochastic approximation and q-learning.
\newblock \emph{Machine learning}, 16\penalty0 (3):\penalty0 185--202, 1994.

\bibitem[Van~Hasselt(2010)]{hasselt2010double}
Van~Hasselt, H.
\newblock Double q-learning.
\newblock \emph{Advances in neural information processing systems},
  23:\penalty0 2613--2621, 2010.

\bibitem[Van~Hasselt et~al.(2016)Van~Hasselt, Guez, and Silver]{van2016deep}
Van~Hasselt, H., Guez, A., and Silver, D.
\newblock Deep reinforcement learning with double q-learning.
\newblock In \emph{Proceedings of the AAAI Conference on Artificial
  Intelligence}, volume~30, 2016.

\bibitem[Van~Hasselt et~al.(2018)Van~Hasselt, Doron, Strub, Hessel, Sonnerat,
  and Modayil]{van2018deep}
Van~Hasselt, H., Doron, Y., Strub, F., Hessel, M., Sonnerat, N., and Modayil,
  J.
\newblock Deep reinforcement learning and the deadly triad.
\newblock \emph{arXiv preprint arXiv:1812.02648}, 2018.

\bibitem[Wang et~al.(2016)Wang, Schaul, Hessel, Hasselt, Lanctot, and
  Freitas]{wang2016dueling}
Wang, Z., Schaul, T., Hessel, M., Hasselt, H., Lanctot, M., and Freitas, N.
\newblock Dueling network architectures for deep reinforcement learning.
\newblock In \emph{International conference on machine learning}, pp.\
  1995--2003. PMLR, 2016.

\bibitem[Watkins \& Dayan(1992)Watkins and Dayan]{watkins1992q}
Watkins, C.~J. and Dayan, P.
\newblock Q-learning.
\newblock \emph{Machine learning}, 8\penalty0 (3-4):\penalty0 279--292, 1992.

\bibitem[Wiering(2000)]{wiering2000multi}
Wiering, M.~A.
\newblock Multi-agent reinforcement learning for traffic light control.
\newblock In \emph{Machine Learning: Proceedings of the Seventeenth
  International Conference (ICML'2000)}, pp.\  1151--1158, 2000.

\bibitem[Xiong et~al.(2018)Xiong, Cao, and Yu]{xiong2018power}
Xiong, R., Cao, J., and Yu, Q.
\newblock Reinforcement learning-based real-time power management for hybrid
  energy storage system in the plug-in hybrid electric vehicle.
\newblock \emph{Applied energy}, 211:\penalty0 538--548, 2018.

\bibitem[Zhang et~al.(2017)Zhang, Pan, and Kochenderfer]{zhang2017weighted}
Zhang, Z., Pan, Z., and Kochenderfer, M.~J.
\newblock Weighted double q-learning.
\newblock In \emph{IJCAI}, pp.\  3455--3461, 2017.

\end{thebibliography}
\bibliographystyle{./icml_formating/icml2021.bst}

\appendix
\onecolumn
\section{Proof of Lemma~\ref{lemma:EE underestimation}} \label{proof:lemma1}
Recall the setting of EE defined as follows:
Let $X = \brac{X_1, \dots ,X_m}$ be a vector of independent random variables and let $\cbrac{\hat{\mu}^{(k)}=\brac{\hat{\mu}_1^{(k)}, \dots, \hat{\mu}_m^{(k)}}}_{k=1}^K$ be $K$ unbiased, independent, estimators of $\dsE[X]$. In our algorithm and analysis, we set $\hat{\mu}^{(j)}_a$ to be the empirical means of the $j^{th}$ subset of samples from the distribution of $X_a$. Also, let $\tilde{k}\in[K]$ be some index. Then, the ensemble estimator is defined as $\hat{\mu}^*_{\EE} = \frac{1}{K-1}\sum_{j \in [K] \backslash \tilde{k}}\hat{\mu}_{\hat{a}^*}^{(j)}$, where $\hat{a}^* = \argmax_a \hat{\mu}_a^{(\tilde{k})}$. 
\SEbias*
\begin{proof}
Since all sets are equally-distributed, assume w.l.o.g. that EE performs its first stage based on the $\tilde{k}^{th}$ subset. Moreover, recall that the estimators $\cbrac{\hat{\mu}^{(k)}}_{k=1}^K$ are independent and unbiased. Therefore, using the tower property, we get
\begin{align*}
    \dsE\sbrac{\hat{\mu}^*_{\EE}} 
    &= \dsE\sbrac{\frac{1}{K-1}\sum_{j \in [K] \backslash \tilde{k}}\hat{\mu}_{\hat{a}^*}^{(j)}} \\
    &= \dsE\sbrac{\dsE\sbrac{\frac{1}{K-1}\sum_{j \in [K] \backslash \tilde{k}}\hat{\mu}_{a}^{(j)}\bigg\vert \hat{a}^*=a} } \tag{Tower property} \\
    &  = \dsE\sbrac{\dsE\sbrac{\frac{1}{K-1}\sum_{j \in [K] \backslash \tilde{k}}\mu_{a}\bigg\vert \hat{a}^*=a} } \tag{Estimators are independent and unbiased}\\
    &  = \dsE\sbrac{\dsE\sbrac{\mu_{a}\vert \hat{a}^*=a} } \\
    & = \dsE\sbrac{\mu_{\hat{a}^*}},
\end{align*}
which proves the first result of the lemma. Next we prove the inequality. First notice that if $P(\hat{a}^* \in \M)=1$, then $\hat{a}^*\in\argmax_a\dsE[X_a]$ almost-surely; then, 
\begin{align*}
    \dsE\sbrac{\hat{\mu}^*_{\EE}} 
    = \dsE\sbrac{\mu_{\hat{a}^*}}
    = \dsE\sbrac{\mu^*} 
    = \mu^*.
\end{align*}
Otherwise, assume that $P(\hat{a}^* \in \M)<1$. Then, there exists an index $a$ such that $\mu_a<\mu^*$ and $P(\hat{a}^*=a)>0$, which implies that:
\begin{align*}
    \dsE\sbrac{\hat{\mu}^*_{\EE}} 
    &= \sum_{a'=1}^m \mu_{a'} P(\hat{a}^*=a')  \\
    &\le \sum_{a'\ne a} \mu^* P(\hat{a}^*=a') + \mu_{a} P(\hat{a}^*=a) \\
    &< \sum_{a'\ne a} \mu^* P(\hat{a}^*=a') + \mu^* P(\hat{a}^*=a) \\
    &= \sum_{a'=1}^m  P(\hat{a}^*=a') \mu^*
    =\mu^*.
\end{align*}



\end{proof}
\clearpage

\section{Proof of Proposition~\ref{prop:proxy}}
\proxyIdentity*
\begin{proof}
Denote by $\hat{a}^*_{\EE}$ and $\hat{a}^*_{\WDE}$, the output indices from the first stage of EE and W-DE, respectively. Importantly, by definitions of both algorithms, we have that $\hat{a}^*_{\WDE} = \argmax_a \hat{\mu}^{(\tilde{k})}_a=\hat{a}^*_{\EE}$. For brevity, we denote this index by $\hat{a}^*$.

For the second stage,  notice that all subsets are of equal-size, i.e., $|S_{\hat{a}^*}^{(j)}|=N/K$ for all $j\in[K]$. Also, denote the sample sets for second stage of the W-DE by $S_a^{\WDE} = \bigcup_{j \in [K]\backslash \tilde{k}} S_{\hat{a}^*}^{(j)}$, which is of size $\abs*{S_a^{\WDE}}=N(1-1/K)$. Then, the ensemble estimator can be written as 
\begin{align*}
    \hat{\mu}_{\EE}^* 
    & = \frac{1}{K-1}\sum_{j \in [K] \backslash \tilde{k}}\hat{\mu}_{\hat{a}^*}^{(j)} \\
    & = \frac{1}{K-1}\sum_{j \in [K] \backslash \tilde{k}} \frac{1}{|S_{\hat{a}^*}^{(j)}|}\sum_{n=1}^{|S_{\hat{a}^*}^{(j)}|}S^{(j)}_{\hat{a}^*}(n)\\ 
    & = \frac{1}{K-1}\sum_{j \in [K] \backslash \tilde{k}} \frac{1}{N/K}\sum_{n=1}^{N/K}S^{(j)}_{\hat{a}^*}(n)\\ 
    & = \frac{1}{N(K-1)/K}\sum_{j \in [K] \backslash \tilde{k}}\sum_{n=1}^{N/K}S^{(j)}_{\hat{a}^*}(n)\\ 
    & = \frac{1}{\abs*{S_a^{\WDE}}}\sum_{n=1}^{\abs*{S_a^{\WDE}}}S_{\hat{a}^*}^{\WDE}(n) \\
    & =  \hat{\mu}_{\hat{a}^*} \brac{\bigcup_{j \in [K]\backslash \tilde{k}} S_{\hat{a}^*}^{(j)}} \\
    & = \hat{\mu}^*_{\WDE}.
\end{align*}

 \end{proof}

\clearpage
\section{MSE Analysis of the Weighted Double Estimator}
\subsection{Calculating the Bias, Variance and MSE of the Estimator}
\label{Ensemble estimator MSE} 



In this appendix, we analyze the bias, variance and MSE of the W-DE as a function of the number of samples used for index estimation $N_1$. Formally, for some $N\in\mathbb{N}$, we define $S^{(1)}$ and $S^{(2)}$ to be sets of samples from the same distribution as $X$ of sizes $\abs{S^{(1)}}=N_1\in[1,N-1]$ and $\abs{S^{(2)}}=N-N_1$ respectively. We further assume that all samples are mutually independent. We then define the empirical mean estimators as 
\begin{align*}
    \hat{\mu}_a^{(j)} = \frac{1}{\abs*{S^{(k)}}} \sum_{j=1}^{\abs*{S^{(k)}}} S_a^{(k)}(j),\quad \forall k\in\cbrac{1,2},\ a\in\sbrac{m}.
\end{align*}
The estimators are naturally unbiased, namely, $\dsE\sbrac{\hat{\mu}_a^{(j)}}=\mu_a$, and since samples are independent, have variance of $\var\sbrac{\hat{\mu}_a^{(j)}}=\frac{1}{\abs{S^{(j)}}}\sigma_a^2$. Then, the W-DE is calculated via the following two-phases estimation procedure:
\begin{enumerate}
    \item \textbf{Index-estimation} using samples from $S^{(1)}$: $\hat{a}^*\in\argmax_a \hat{\mu}_a^{(1)}$.
    \item \textbf{Mean-estimation} using samples from $S^{(2)}$: $\hat{\mu}^*_{\WDE} = \hat{\mu}_{\hat{a}^*}^{(2)}$.
\end{enumerate}
By the proxy-identity (\Cref{prop:proxy}), this estimator is closely related to the ensemble estimator. Specifically, we can understand how the ensemble size affects the MSE of EE by analyzing the effect of $N_1$ on the MSE of W-DE. 
We now directly calculate the statistics of W-DE. 
\begin{itemize}
    \item \textbf{First moment.} Since the samples from the different components of $X$ are mutually independent and $\hat{\mu}^{(2)}$ is an unbiased estimator, we can write
    \begin{align*}
        \dsE\sbrac{\hat{\mu}^*_{\WDE}} 
        = \dsE\sbrac{\hat{\mu}_{\hat{a}^*}^{(2)}}
        = \dsE\sbrac{\dsE\sbrac{\hat{\mu}_{a}^{(2)}\vert \hat{a}^*=a}}
        = \dsE\sbrac{\dsE\sbrac{\mu_{a}\vert \hat{a}^*=a}}
        = \dsE\sbrac{\mu_{\hat{a}^*}}
        = \sum_{a=1}^m \mu_a P\brac{\hat{a}^*=a}.
    \end{align*}
    \item \textbf{Second moment.} Similar to the first moment, we get
    \begin{align*}
        \dsE\sbrac{\brac{\hat{\mu}^*_{\WDE}}^2} 
        = \dsE\sbrac{\brac{\hat{\mu}_{\hat{a}^*}^{(2)}}^2}
        &= \dsE\sbrac{\dsE\sbrac{\brac{\hat{\mu}_{a}^{(2)}}^2\vert \hat{a}^*=a}} \\
        &= \sum_{a=1}^m\dsE\sbrac{\brac{\hat{\mu}_{a}^{(2)}}^2}P\brac{\hat{a}^*=a}  \\
        &= \sum_{a=1}^m\brac{\var\brac{\hat{\mu}_{a}^{(2)}} + \brac{\dsE\sbrac{\hat{\mu}_{a}^{(2)}}}^2}P\brac{\hat{a}^*=a} \\ 
        &= \sum_{a=1}^m\brac{\frac{\sigma_a^2}{N-N_1} + \mu_a^2}P\brac{\hat{a}^*=a} .
    \end{align*}
\end{itemize}
Next, we calculate the bias, variance and MSE of the estimator. To this end, assume w.l.o.g. that $\mu_1\ge\dots\ge \mu_m$. In particular, it implies that $\mu^*=\mu_1$. 
\begin{itemize}
    \item \textbf{Bias. } 
    \begin{align}
    \bias\brac{\hat{\mu}^*_{\WDE}} 
    = \dsE\sbrac{\hat{\mu}^*_{\WDE}} - \mu^* 
    = \sum_{a=1}^m \mu_a P\brac{\hat{a}^*=a} - \mu_1
    = \sum_{a=1}^m (\mu_a-\mu_1) P\brac{\hat{a}^*=a}. \label{eq: W-DE bias}
    \end{align}
    \item \textbf{Variance. }
    \begin{align}
    \var\brac{\hat{\mu}^*_{\WDE}} 
    &=\dsE\sbrac{\brac{\hat{\mu}^*_{\WDE}}^2}  - \brac{\dsE\sbrac{\hat{\mu}^*_{\WDE}}}^2
    = \sum_{a=1}^m\brac{\frac{\sigma_a^2}{N-N_1} + \mu_a^2}P\brac{\hat{a}^*=a} - \brac{\sum_{a=1}^m \mu_a P\brac{\hat{a}^*=a}}^2. \label{eq: W-DE var}
    \end{align}
\end{itemize}
Finally, we calculate the MSE of the estimator. First recall that for any estimator $\hat\mu$ of $\mu^*$, it holds that 
\begin{align} 
    \MSE\brac{\hat\mu}
    &= \dsE\sbrac{\brac{\hat\mu-\mu^*}^2}\nonumber\\
    &= \dsE\sbrac{\brac{\hat\mu - \dsE\sbrac{\hat\mu} + \dsE\sbrac{\hat\mu}-\mu^*}^2}\nonumber\\
    &= \underbrace{\dsE\sbrac{\brac{\hat\mu - \dsE\sbrac{\hat\mu}}^2}}_{=\var(\hat\mu)} + 2\underbrace{\dsE\sbrac{\hat\mu - \dsE\sbrac{\hat\mu}}}_{=0}\brac{\dsE\sbrac{\hat\mu}-\mu^*} + \underbrace{\brac{\dsE\sbrac{\hat\mu}-\mu^*}^2}_{=\bias(\hat\mu)^2} \nonumber\\
    & = \var(\hat\mu) + \bias(\hat\mu)^2 \label{eq:general MSE}
\end{align}
Specifically, for W-DE, we get
\begin{align}
    \bias\brac{\hat{\mu}^*_{\WDE}}^2
     &= \brac{\sum_{a=1}^m (\mu_a-\mu_1) P\brac{\hat{a}^*=a}}^2 \nonumber\\
     & = \brac{\sum_{a=1}^m \mu_aP\brac{\hat{a}^*=a}}^2 -2\sum_{a=1}^m \mu_a\mu_1 P\brac{\hat{a}^*=a} + \mu_1^2 \nonumber\\
     & = \brac{\sum_{a=1}^m \mu_aP\brac{\hat{a}^*=a}}^2 + \sum_{a=1}^m (\mu_1^2 -2\mu_a\mu_1) P\brac{\hat{a}^*=a} \label{eq:W-DE bias}
\end{align}
and combining \cref{eq: W-DE var}, \cref{eq:general MSE} and \cref{eq:W-DE bias}, we get
\begin{align}
    \MSE\brac{\hat{\mu}^*_{\WDE}} 
    &= \var\brac{\hat{\mu}^*_{\WDE}} + \bias\brac{\hat{\mu}^*_{\WDE}}^2 \nonumber\\
    &= \sum_{a=1}^m\brac{\frac{\sigma_a^2}{N-N_1} + \mu_a^2}P\brac{\hat{a}^*=a} + \sum_{a=1}^m (\mu_1^2 -2\mu_a\mu_1) P\brac{\hat{a}^*=a} \nonumber\\
    & = \sum_{a=1}^m\brac{\frac{\sigma_a^2}{N-N_1} + (\mu_1-\mu_a)^2}P\brac{\hat{a}^*=a} \nonumber\\
    & \triangleq \sum_{a=1}^m\brac{\frac{\sigma_a^2}{N-N_1} + \Delta_a^2}P\brac{\hat{a}^*=a}, \label{eq:w-de mse}
\end{align}
where we defined $\Delta_a = \mu_1-\mu_a$.

Finally, we address the values of the probabilities $P\brac{\hat{a}^*=a}$ where $m=2$. 
For this, we assume that all samples are drawn independently from a Gaussian distribution. Then, if $\Phi$ is the cumulative distribution function (cdf) of a standard Gaussian variable, we have that 
\begin{align*} 
    P(\hat{a}^*=1) 
    &= P\brac{\hat{\mu}^{(1)}_1 \ge \hat{\mu}^{(1)}_2} \\
    &= P(Z_{1,2} \geq 0) \tag{For $Z_{a,j}\sim \N\brac{\mu_a-\mu_j, \frac{\sigma_a^2 + \sigma_j^2}{N_1}}$}  \\
    &= 2\cdot \prod_{j=1}^2 \Phi \brac{\frac{\sqrt{N_1}\brac{\mu_1- \mu_j}}{\sqrt{\sigma_j^2 + \sigma_1^2}}} 
\end{align*}
where in the last relation, we substituted the cdf of the Gaussian distribution and used the fact that for $j=a$, we have $\Phi(0)=\frac{1}{2}$. Moreover, from symmetry, a similar relation holds for $P(\hat{a}^*=2)$, and we can write 
\begin{align} 
    P(\hat{a}^*=a) 
    = 2\cdot \prod_{j=1}^2 \Phi \brac{\frac{\sqrt{N_a}\brac{\mu_a- \mu_j}}{\sqrt{\sigma_j^2 + \sigma_a^2}}}. \label{eq:prob max at a}
\end{align}
Finally, plugging \cref{eq:prob max at a} to \cref{eq:w-de mse}, we get
\begin{align} 
        \MSE\brac{\hat{\mu}^*_{\WDE}} 
        = 2\sum_{i=1}^2 \sbrac{\brac{ \frac{\sigma_a^2}{N-N_1}+ \brac{\mu_a - \mu_1}^2} \prod_{j=1}^2\Phi \brac{\frac{\sqrt{N_i}\brac{\mu_i- \mu_j}}{\sqrt{\sigma_i^2 + \sigma_j^2}}}} \label{eq:final_mse}
\end{align}

\subsection{Proof of Proposition~\ref{prop:split_ratio}}
\label{proof:prop2}
\propSplitRatio*
\begin{proof}
For this proof, we analyze the case of $m=2$ where $\sigma_1=\sigma_2=\sigma$ and $\Delta = \mu_1-\mu_2\ge0$. In this specific case, by \cref{eq:final_mse}, we have
\begin{align*} 
    \MSE\brac{\hat\mu^*_{\WDE}}&= \frac{\sigma^2}{N-N_1} \Phi \brac{\frac{\Delta\sqrt{N_1}}{\sqrt{2}\sigma}} + \brac{\frac{\sigma^2}{N-N_1}+\Delta^2}\brac{1-\Phi \brac{\frac{\Delta \sqrt{N_1}}{\sqrt{2}\sigma}}} \nonumber\\
    &= \frac{\sigma^2}{N-N_1}+\Delta^2 -\Delta^2 \Phi\brac{\frac{\Delta\sqrt{N_1}}{\sqrt{2}\sigma}} \nonumber
    \triangleq \MSE(N_1)
\end{align*}
For the analysis, we allow $N_1$ to be continuous in the interval $[1,N-1]$. We then show that there exists a value $\tilde{N_1}$ such that for any $N\ge \tilde{N}_1$, the function $\MSE(N_1)$ is strictly increasing. This implies that even for integer values, any minimizer of the MSE $N_1^*$ is upper-bounded by $N_1^*\le\ceil{\tilde{N}_1}$. To do so, we lower bound the derivative of the MSE, which equals to 
\begin{align} 
    \frac{d\MSE(N_1)}{dN_1} 
    &= \frac{\sigma^2}{\brac{N-N_1}^2} - \frac{\Delta^3}{4\sigma\sqrt{\pi N_1}}e^{-\frac{1}{4}\frac{\Delta^2}{\sigma^2}N_1} \label{eq:mse_derivative} \\
    & = \frac{\sigma^2}{\brac{N-N_1}^2} - \brac{\frac{2\sigma^2}{\sqrt{\pi}N_1^2}}\brac{\frac{1}{4}\brac{\frac{\Delta}{\sigma}}^2 N_1}^{\frac{3}{2}} e^{-\frac{1}{4}\brac{\frac{\Delta}{\sigma}}^2 N_1}\label{eq:mse_derivative exp}
\end{align}
\textbf{Proof of part (1)}
We now focus on lower bounding the derivative in \cref{eq:mse_derivative}. Denoting $b=\frac{\Delta^3}{4\sigma \sqrt{\pi}}$, we have that 
\begin{align*} 
    \frac{d\MSE(N_1)}{dN_1} 
    &\ge \frac{\sigma^2}{\brac{N-N_1}^2} - \frac{\Delta^3}{4\sigma\sqrt{\pi N_1}}\\
    &= \frac{\sigma^2}{\brac{N-N_1}^2} - \frac{b}{\sqrt{ N_1}} \\
    & = \frac{\sigma^2\sqrt{N_1} - b \brac{N^2 - 2N N_1 + N_1^2}}{(N-N_1)^2}\\
    & > \frac{\sigma^2\sqrt{N_1} - b N\brac{N - N_1}}{(N-N_1)^2}. \tag{$N_1<N$ and thus $N_1^2<N_1N$}
\end{align*}
Next, denote $x=\sqrt{N_1}$. Then, the numerator of the derivative is quadratic in $x$ and can be written as 
\begin{align*}
    \sigma^2\sqrt{N_1} - b N\brac{N - N_1} = \sigma^2x - b N\brac{N-x^2} = bNx^2 + \sigma^2x -bN^2.
\end{align*}
as the denominator is always positive, and since $b>0$ and $x=\sqrt{N_1}>0$, the derivative will always be strictly positive for any $N_1>\tilde{N}_1$ such that
\begin{align*}
    \sqrt{\tilde{N}_1} &\ge  \frac{-\sigma^2 + \sqrt{\sigma^4 + 4b^2N^3}}{2 b N} = 
\frac{-\sigma^2 + \sqrt{\sigma^4 + 4\brac{\frac{\Delta^3}{4\sigma \sqrt{\pi}}}^2N^3}}{2N\frac{\Delta^3}{4\sigma \sqrt{\pi}}}
= \sqrt{N}\frac{-1 + \sqrt{1 + \brac{\frac{1}{2\sqrt{\pi}}}^2 \brac{\frac{\Delta}{\sigma / \sqrt{N}}}^6}} {\brac{\frac{\Delta}{\sigma/\sqrt{N}}}^3 \frac{1}{2\sqrt{\pi}}}\\
&\overset{(*)}{=} \sqrt{N}\frac{-1 + \sqrt{1 + c^2 \brac{\frac{\Delta}{\sigma / \sqrt{N}}}^6}} {\brac{\frac{\Delta}{\sigma/\sqrt{N}}}^3 c} = \sqrt{N}\frac{-1 + \sqrt{1 + c^2 \brac{\SNR}^6}} {\brac{\SNR}^3 c}
\end{align*}
where in $(*)$ we defined $c\triangleq\frac{1}{2\sqrt{\pi}}$ and in the last equality, we substituted $\SNR \triangleq\frac{\Delta}{\sigma / \sqrt{N}}$. Notably, when $\SNR\to0$, one can easily verify that $\tilde{N_1}\to0$. Thus, when the SNR is large enough, we will have that $\tilde{N_1}<1$, and $\MSE(N_1)$ will strictly increase in $N_1$ for any integer value in $[1,N-1]$; then, it will be minimized at $N_1^*=1$.



\textbf{Proof of part (2):} 
For this part of the proof, we focus on the second form of the derivative in \cref{eq:mse_derivative exp} and denote $f(x) = x^{3/2}e^{-x}$. Importantly, there exists $x_0$ such that $f(x)< e^{-x/2}$ for any $x\ge x_0$. Since $N_1\ge1$ for any $\Delta,\sigma$ such that $\frac{1}{4}\brac{\frac{\Delta}{\sigma}}^2\ge x_0$, we have that 
\begin{align*} 
    \frac{d\MSE(N_1)}{dN_1} 
    &= \frac{\sigma^2}{\brac{N-N_1}^2} - \brac{\frac{2\sigma^2}{\sqrt{\pi}N_1^2}}\brac{\frac{1}{4}\brac{\frac{\Delta}{\sigma}}^2 N_1}^{\frac{3}{2}} e^{-\frac{1}{4}\brac{\frac{\Delta}{\sigma}}^2 N_1} \\
    &\ge \frac{\sigma^2}{\brac{N-N_1}^2} - \brac{\frac{2\sigma^2}{\sqrt{\pi}N_1^2}}e^{-\frac{1}{8}\brac{\frac{\Delta}{\sigma}}^2 N_1}\\
    &\ge \frac{\sigma^2}{\brac{N-N_1}^2} - \brac{\frac{2\sigma^2}{\sqrt{\pi}N_1^2}}e^{-\frac{1}{8}\brac{\frac{\Delta}{\sigma}}^2}. \tag{$N_1\ge1$}
\end{align*}
One can easily verify that this lower bound on the derivative of the MSE is strictly positive for any $N_1\in(\tilde{N_1},N)$, where 
\begin{align*}
    \tilde{N_1} 
    = \frac{\sqrt{\brac{\frac{2}{\sqrt{\pi}}}e^{-\frac{1}{8}\brac{\frac{\Delta}{\sigma}}^2}}}{1 + \sqrt{\brac{\frac{2}{\sqrt{\pi}}}e^{-\frac{1}{8}\brac{\frac{\Delta}{\sigma}}^2}}}N
    = \frac{\sqrt{\brac{\frac{2}{\sqrt{\pi}}}e^{-\frac{1}{8N}\SNR^2}}}{1 + \sqrt{\brac{\frac{2}{\sqrt{\pi}}}e^{-\frac{1}{8N}\SNR^2}}}N
\end{align*}
Notably, for any fixed $N$ and large enough SNR, we have that $\frac{1}{4}\brac{\frac{\Delta}{\sigma}}^2\ge x_0$ and $\tilde{N}_1<1$. Then, the MSE is strictly increasing in $N_1$ for any integer value $N_1\in[1,N]$ and $N_1^*=1$.





\textbf{Proof of part (3):} We start from the derivative form of \cref{eq:mse_derivative exp}. Notice that the strict maximum of the function $f(x) = x^{3/2}e^{-x}$ over $[0,\infty)$ is achieved at $x^*=1.5$; hence, for any $x\ge0$, we have that $f(x)\le f(x^*)=1.5^{1.5}e^{-1.5}$, which allows us to bound the derivative by 
\begin{align}\label{eq:mse_derivative_lower_bound_N1<N/2}
    \frac{d\MSE(N_1)}{dN_1}  \geq  \frac{\sigma^2}{\brac{N-N_1}^2} - \frac{2\sigma^2}{\sqrt{\pi}N_1^2} 1.5^{1.5}e^{-1.5}
    \triangleq\frac{\sigma^2}{\brac{N-N_1}^2} - \frac{\sigma^2}{N_1^2} d,
\end{align}
where $d=1.5^{1.5}e^{-1.5} \frac{2}{\sqrt{\pi}}$. One can easily verify that the lower bound of \cref{eq:mse_derivative_lower_bound_N1<N/2} equals zero when $N_1 =\frac{\sqrt{d}}{1+\sqrt{d}}N\triangleq \tilde{N_1}$ and is strictly positive for any $N_1\in(\tilde{N_1},N)$.
Notably, $\tilde{N}_1\approx 0.405N$, and therefore, if $N>10$ then $\ceil{\tilde{N_1}}<N/2$, and choosing $N_1 = \ceil{\tilde{N_1}}$ leads to strictly lower MSE than choosing $N_1=N/2$. This hold for any values of $\Delta,\sigma>0$ and proves the third claim of the proposition.
\end{proof}

\clearpage
\section{Single Estimator (SE) is positively biased} \label{proof:se_unbiased}
Let $X = \brac{X_1, \dots , X_m}$ be a vector of $m$ independent random variables with expectations $\mu= \brac{\mu_1, \dots \mu_m}$. We wish to estimate $\max_a\dsE\sbrac{X_a} = \max_a \mu_a$ using samples from same distribution: $S = \cbrac{S_1, \dots S_m}$ where $S_a=\cbrac{S_a(n)}_{n=1}^{\abs{S_a}}$ is a set of i.i.d.\ samples from the same distribution as $X_a$.

The \textit{Single Estimator}- $\hat{\mu}^*_{\SE}$ approximates the maximal expectation via the maximal empirical mean over the entire set of samples; namely, if $\hat{\mu}_a = \frac{1}{|S_a|}\sum_{n=1}^{|S_a|}S_a(n)$, then
\begin{align} \label{eq:SE}
    \hat{\mu}^*_{\SE}(S) \triangleq \max_{a\in[m]}\hat{\mu}_a \approx \max_{a\in[m]}\dsE\sbrac{X_a}.
\end{align}
The single estimator presented in \eqref{eq:SE} overestimates the true maximal expectation. i.e., is positively biased: $\dsE\sbrac{\hat{\mu}^*_{\SE}} - \max_a \dsE X_a \geq 0$.
\begin{proof}
Let $a^*\in\argmax_{a\in[m]}\mu_a$ be an index that maximizes the expected value of $X$ and let $\hat{a}^* = \argmax_{a\in[m]} \hat{\mu}_a$ be an index that maximizes the empirical mean. 
Specifically, by definition, we have that $\hat{\mu}_{a^*} \le \hat{\mu}_{\hat{a}^*}$. In turn, this implies that
\begin{align*} \label{eq:SE inequality}
    \mu_{a^*} - \hat{\mu}_{\hat{a}^*}
    \leq \mu_{a^*} - \hat{\mu}_{a^*}.
\end{align*}
Then, by the monotonicity of the expectation, we get
\begin{align*}
    \dsE \sbrac{\mu_{a^*} - \hat{\mu}_{\hat{a}^*}} \leq \dsE \sbrac{\mu_{a^*} - \hat{\mu}_{a^*}} = 0,
\end{align*}
or $\dsE\sbrac{\mu_{a^*}} \le \dsE \sbrac{\hat{\mu}_{\hat{a}^*}}$, which concludes the proof. The equality holds since we assumed that the samples $S_a(j)$ are drawn from the same distribution as $X_a$ and are therefore unbiased:
\begin{align*}
    \dsE[\hat{\mu}_{a^*}] 
    = \dsE\sbrac{\frac{1}{|S_{a^*}|}\sum_{j=1}^{|S_{a^*}|}S_{a^*}(j)}
    = \mu_{a^*}.
\end{align*}

\end{proof}

\clearpage
\section{Implementation details}
\label{appendix:implementation}
\subsection{Meta-Chain MDP}
\textbf{Environment:} We have tested the algorithms (QL, DQL and \alg{}) over 5000 episodes and averaged the results of 50 random-seeds. The meta-chain MDP was constructed by 6 chain MDPs with $r(D^i)\sim \N(\mu_i, 1)$, where $\cbrac{\mu_i}_{i=1}^6 = \cbrac{-0.6, -0.4, -0.2, 0.2, 0.4, 0.6}$. In all chain MDPs, the number of actions available at state $B^i$ was set to 10.

\textbf{Algorithm:} The learning rate $\alpha_t$ of \alg{} was chosen to be polynomial, similar to the one in \cite{hasselt2010double}; $\alpha(s_t,a_t) = 1/n^i_t(s_t,a_t)^{0.8}$, where $n^i_t(s,a)$ is the number of updates made until time $t$ in the $(s,a)$ entry of the $i$'th estimator. 
We tested \alg{} for different ensemble sizes $K \in \cbrac{3,7,10,15,25}$. The discount factor was set to $\gamma=1$.

\subsection{Deep EBQL}
Our code is based on \citep{kaixin}, and can be found in the supplementary material of our paper.

Denote the TD-error of \alg{} for ensemble member $k$ as $Y_k^{\alg{}}(s,a,r,s') = r + \gamma Q^{EN \setminus k}(s',\hat{a}^*) - Q^{k}(s,a)$.
\begin{algorithm}[H]
  \caption{Deep \algfull{} (Deep-\alg)}
  \label{alg:deep_EBQL}
\begin{algorithmic}
  \STATE {\bfseries Initialize:} Q-ensemble with size $K$, parametrized by random weights: $\cbrac{Q(\cdot,\ \cdot \  ;\theta_i)}_{i=1}^K$, Experience Buffer $\B$
  \FOR{$t = 0, \dots , T$}
    \STATE Choose action $a_t = \argmax_{a}
    \sbrac{\sum_{i=1}^K Q^i(s_t,a ;\theta_i)}$ 
    \STATE $a_t = \text{explore}(a_t)$ \hfill\COMMENT{//e.g. $\epsilon$-greedy}
    \STATE $s_{t+1}, r_t \gets \text{env.step}(s_t,a_t)$
    \STATE $\B \leftarrow (s_t,a_t,r_t,s_{t+1})$
    \STATE Set $k = t \% K$  
    \STATE Define $a^* = \argmax_a Q(s_{t+1},a ; \theta_k)$
    \STATE Take a gradient step to minimize $\mathcal{L} = \dsE_{(s,a^*,r,s')\sim \B}\norm*{Y_k^{\alg{}}(s,a,r,s') - Q(s,a ; \theta_k)}^2$
  \ENDFOR
  \STATE \textbf{Return} $\cbrac{Q^i}_{i=1}^K$
\end{algorithmic}
\end{algorithm}
While in theory, each ensemble member should observe a unique data set,  
we saw that similar to \citet{osband2016deep, osband2018randomized, anschel2017averaged, lee2020sunrise}, \alg{} performs well when \textit{all} ensemble members are updated \textit{every} iteration using the same samples from the experience replay. In addition, we use a shared feature extractor for the entire ensemble, followed by the independent  head for each ensemble member.

\subsection{The Double Q-Learning Algorithm}
\label{DQL_alg}
\begin{algorithm}[H]
  \caption{Double Q-Learning (DQL)}
  \label{alg:DQL}
\begin{algorithmic}
  \STATE {\bfseries Initialize:} Two Q-tables: $Q^A$ and $Q^B$ , $s_0$ 
  \FOR{$t= 1, \dots , T$}
    \STATE Choose action $a_t=\argmax_{a}
    \sbrac{Q^A(s_t,a) + Q^B(s_t,a)}$ 
    \STATE $a_t= \text{explore}(a_t)$\hfill\COMMENT{//e.g. $\epsilon$-greedy}
    \STATE $s_{t+1}, r_t \gets \text{env.step}(s_t,a_t)$
    \IF {$t \% 2 == 0$} 
    \STATE Define $a^* = \argmax_a Q^A(s_{t+1},a)$\hfill\COMMENT{//Update A}
    \STATE $Q^A(s_t,a_t) \gets Q^A(s_t,a_t) + \alpha_t\brac{r_t + \gamma Q^B(s_{t+1},a^*) -Q^A(s_t,a_t)}$
    \ELSE
    \STATE Define $b^* = \argmax_a Q^B(s_{t+1},a)$\hfill\COMMENT{//Update B}
    \STATE $Q^B(s_t,a_t) \gets Q^B(s_t,a_t) + \alpha_t\brac{r_t + \gamma Q^A(s_{t+1},b^*) -Q^B(s_t,a_t)}$
    \ENDIF
  \ENDFOR
  \STATE \textbf{Return} $Q^A$, $Q^B$
\end{algorithmic}
\end{algorithm}

\section{Additional Experimental Results}\label{appendix:experimental results}

Below we provide the per-environment training curves. A somewhat surprising result is that \alg{} is capable of outperforming Rainbow in several domains, although Rainbow combines a multitude of improvements which we did not add on top of \alg{} (distributional, dueling, prioritized experience replay and multi-step learning).

\begin{figure}[H]
    \centering
    \includegraphics[width=0.24\linewidth]{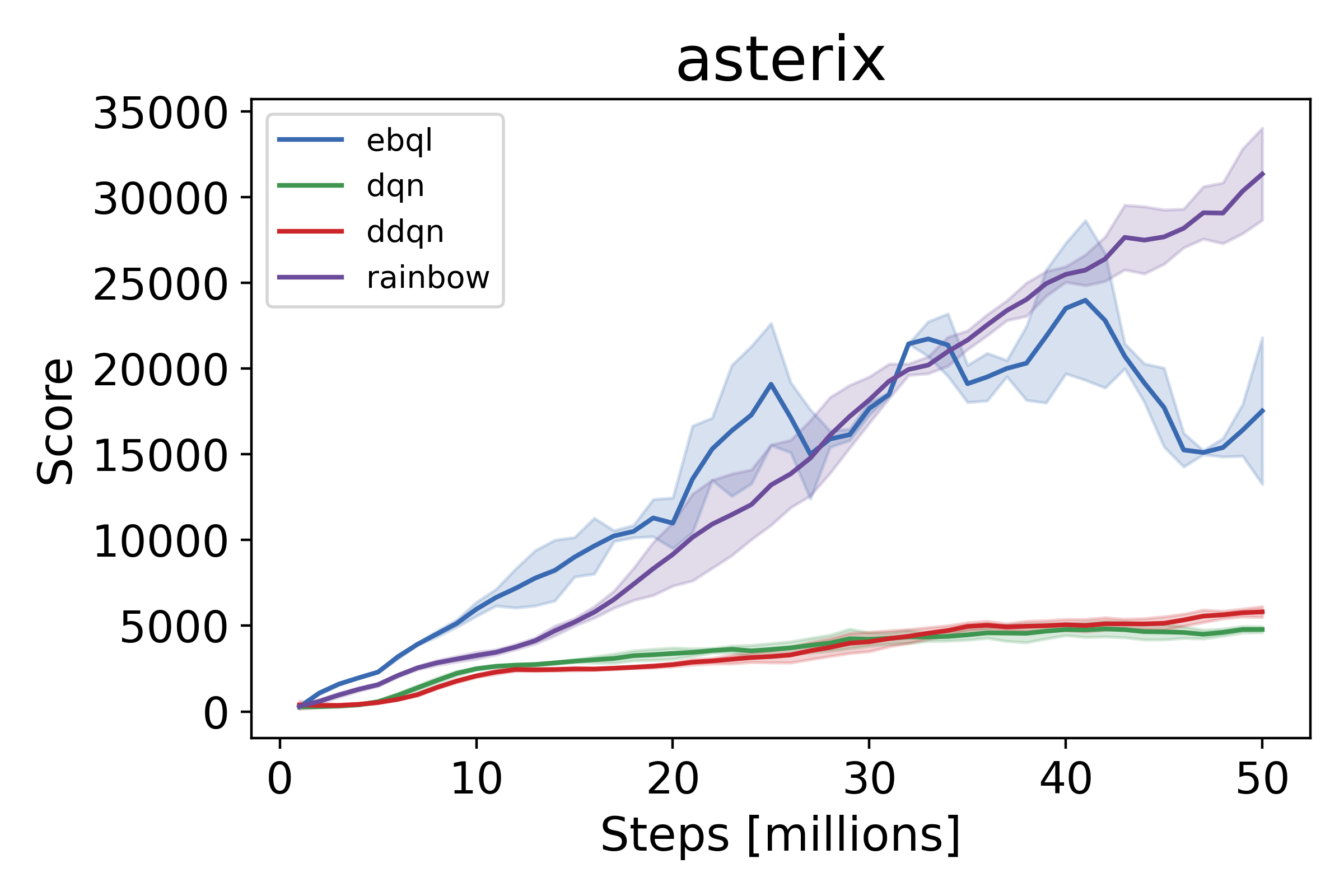}
    \includegraphics[width=0.24\linewidth]{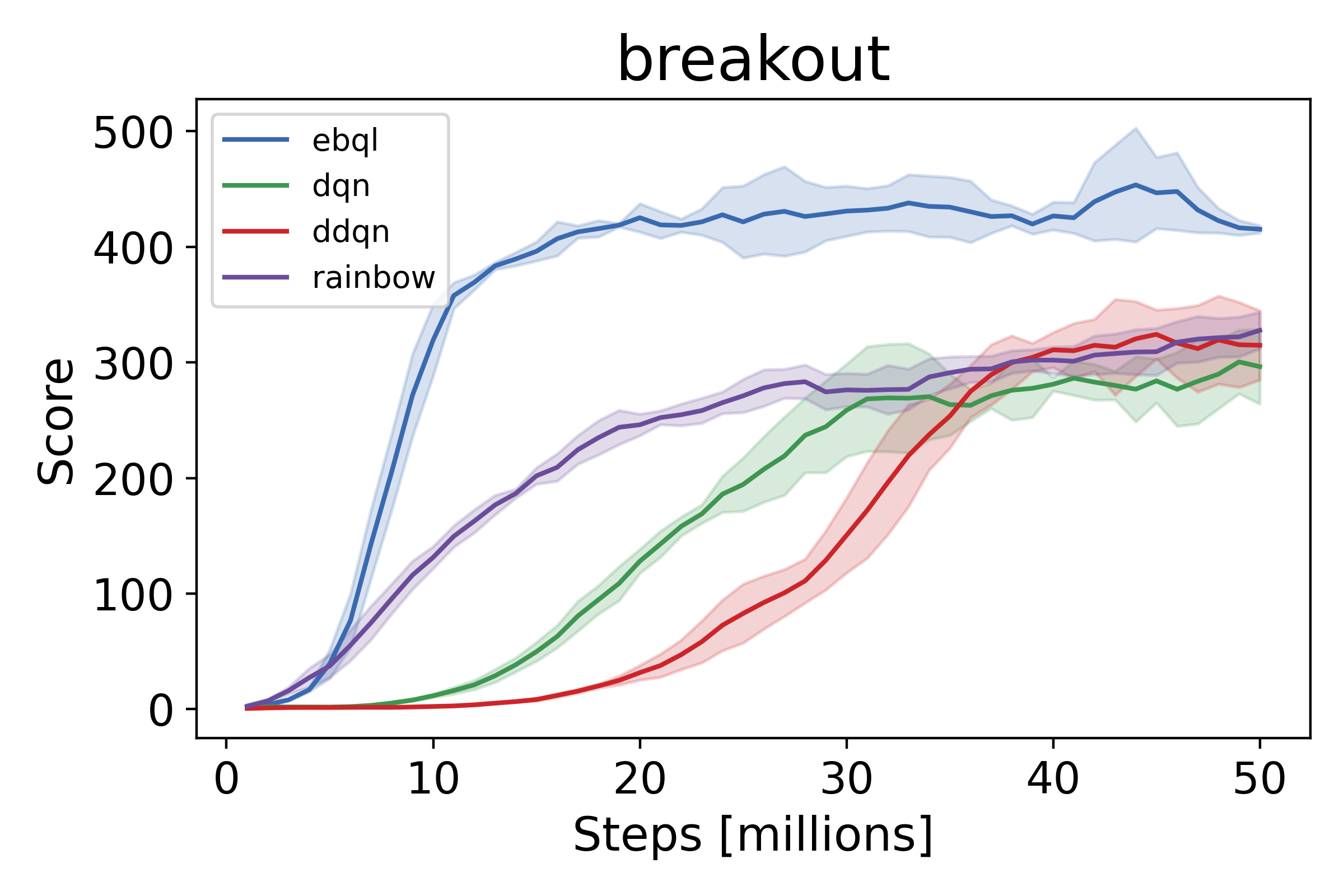}
    \includegraphics[width=0.24\linewidth]{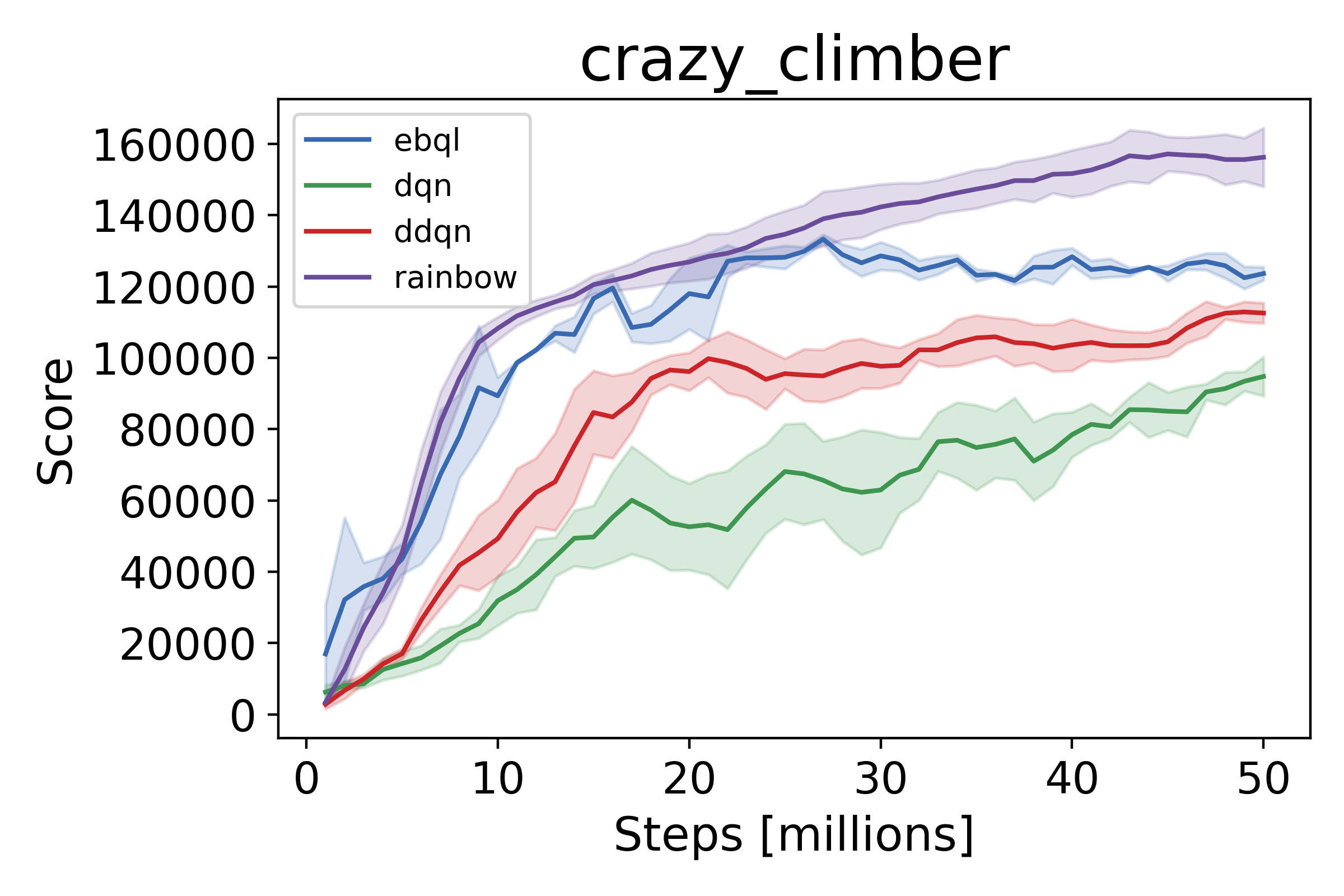}
    \includegraphics[width=0.24\linewidth]{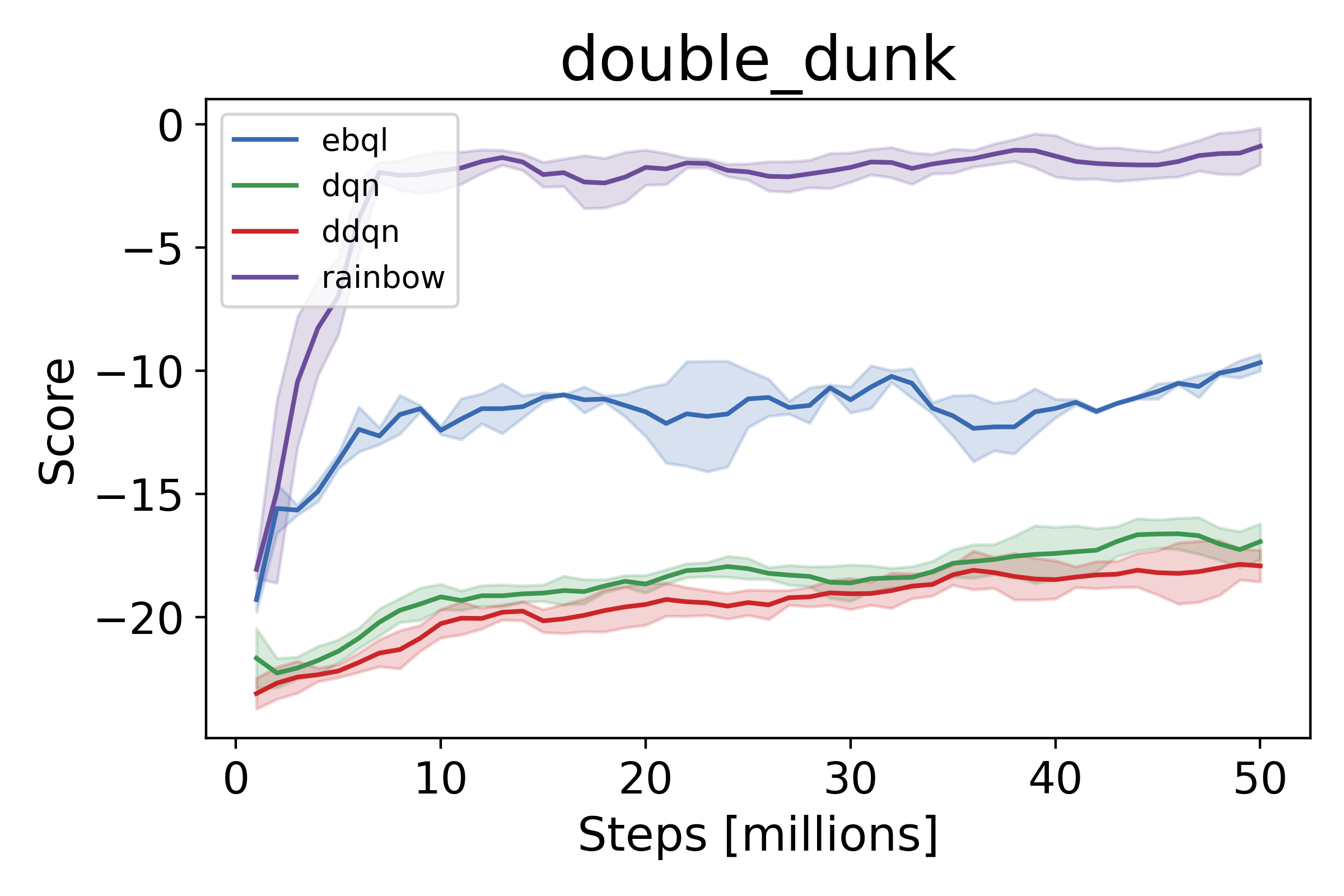}
    \includegraphics[width=0.24\linewidth]{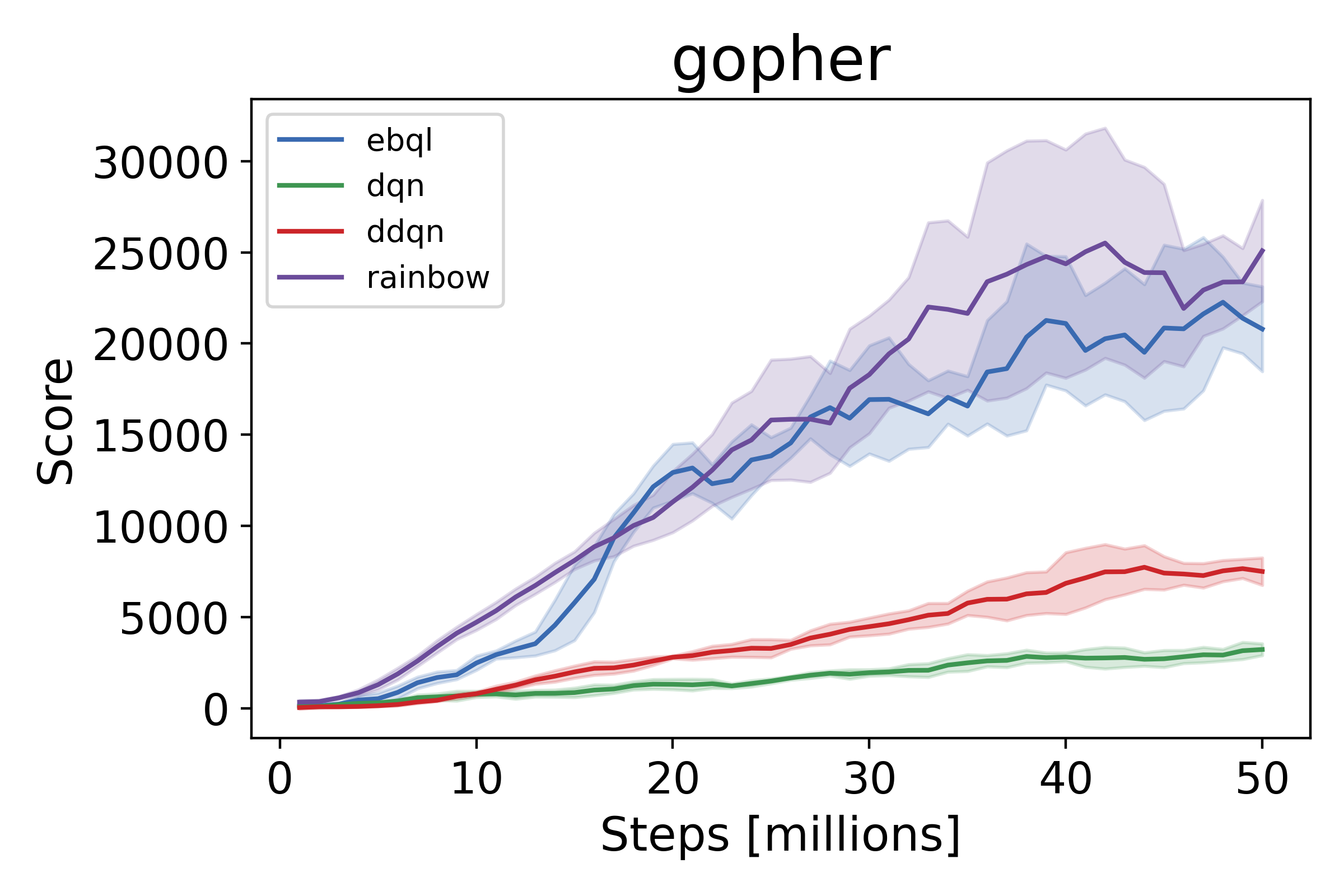}
    \includegraphics[width=0.24\linewidth]{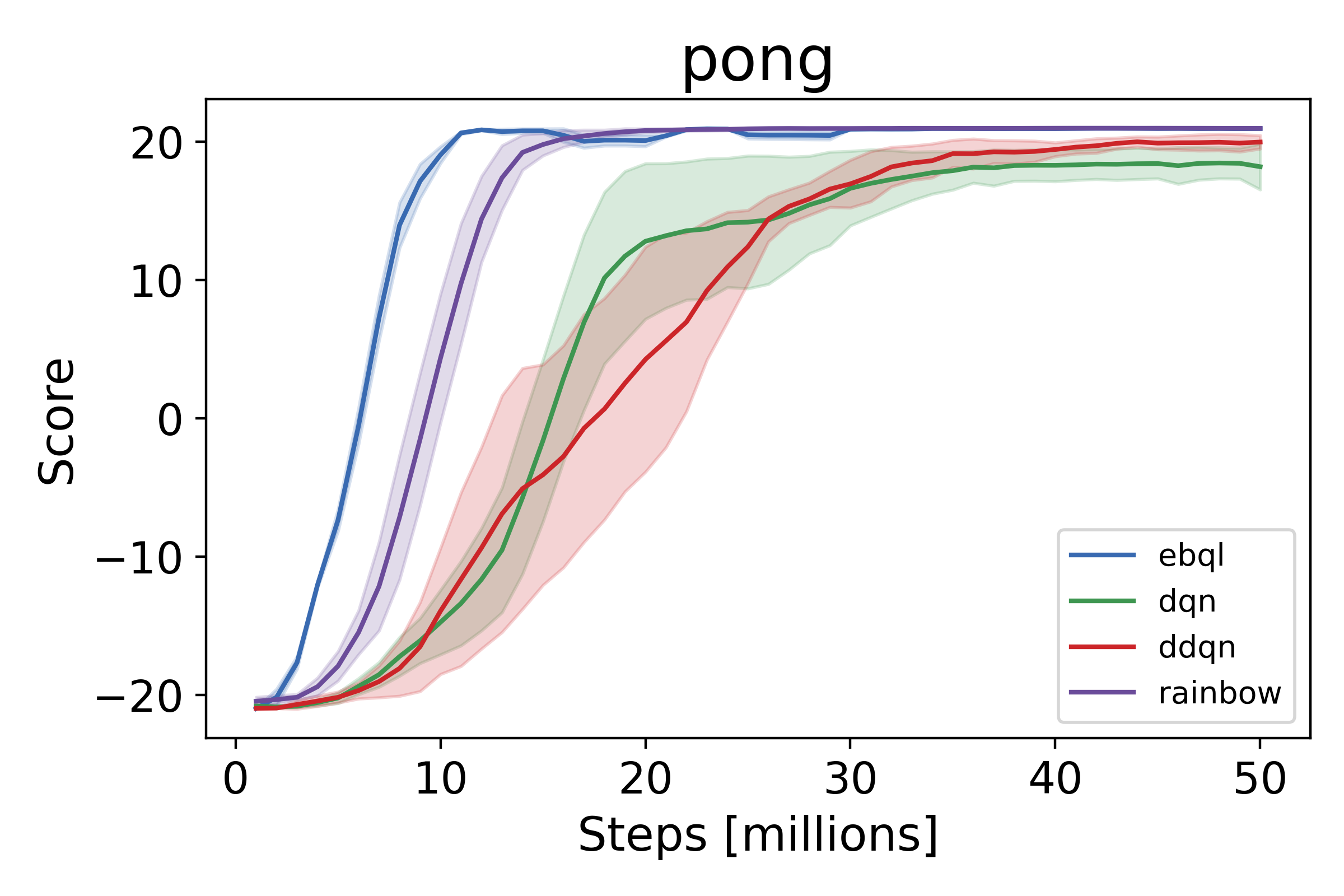}
    \includegraphics[width=0.24\linewidth]{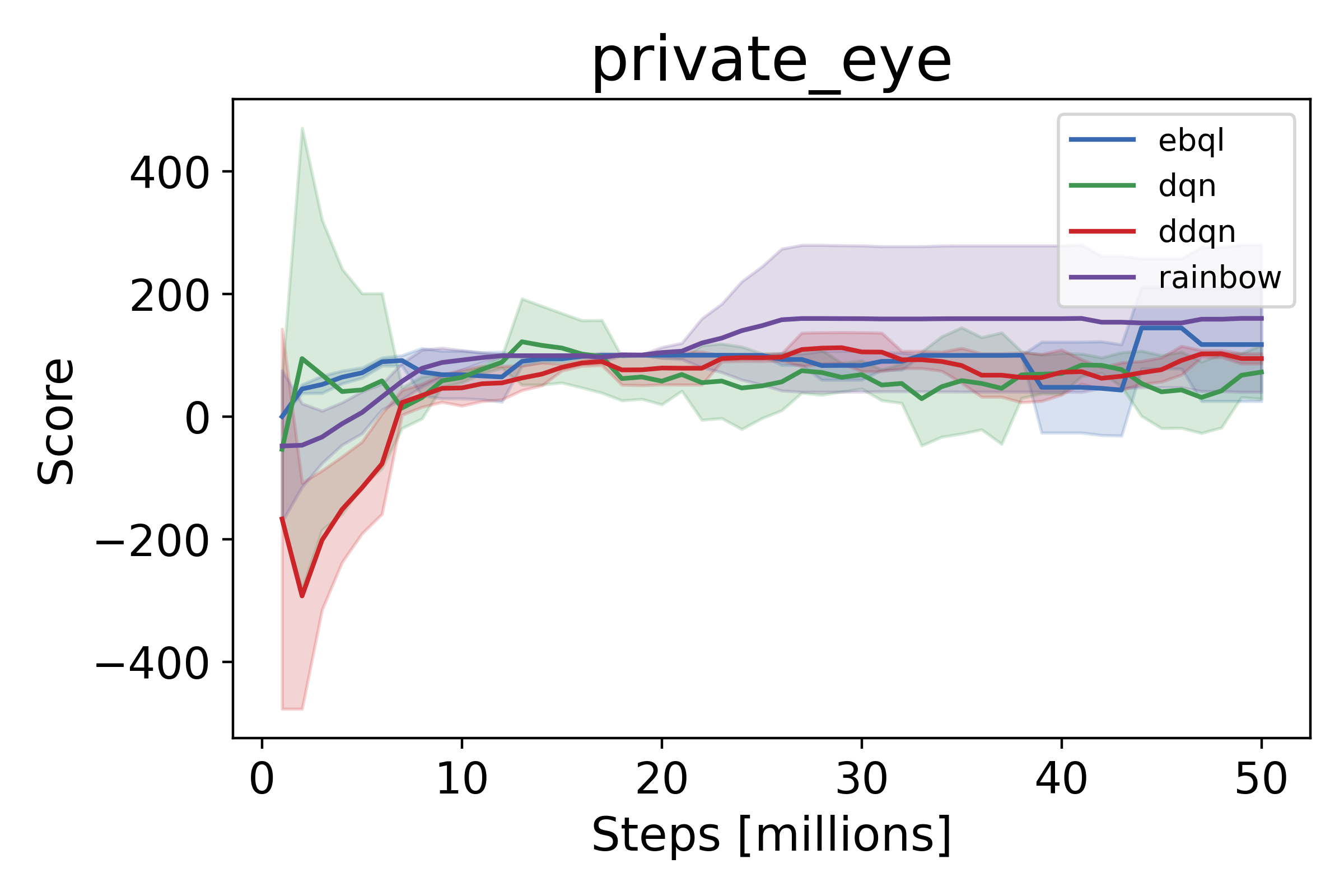}
    \includegraphics[width=0.24\linewidth]{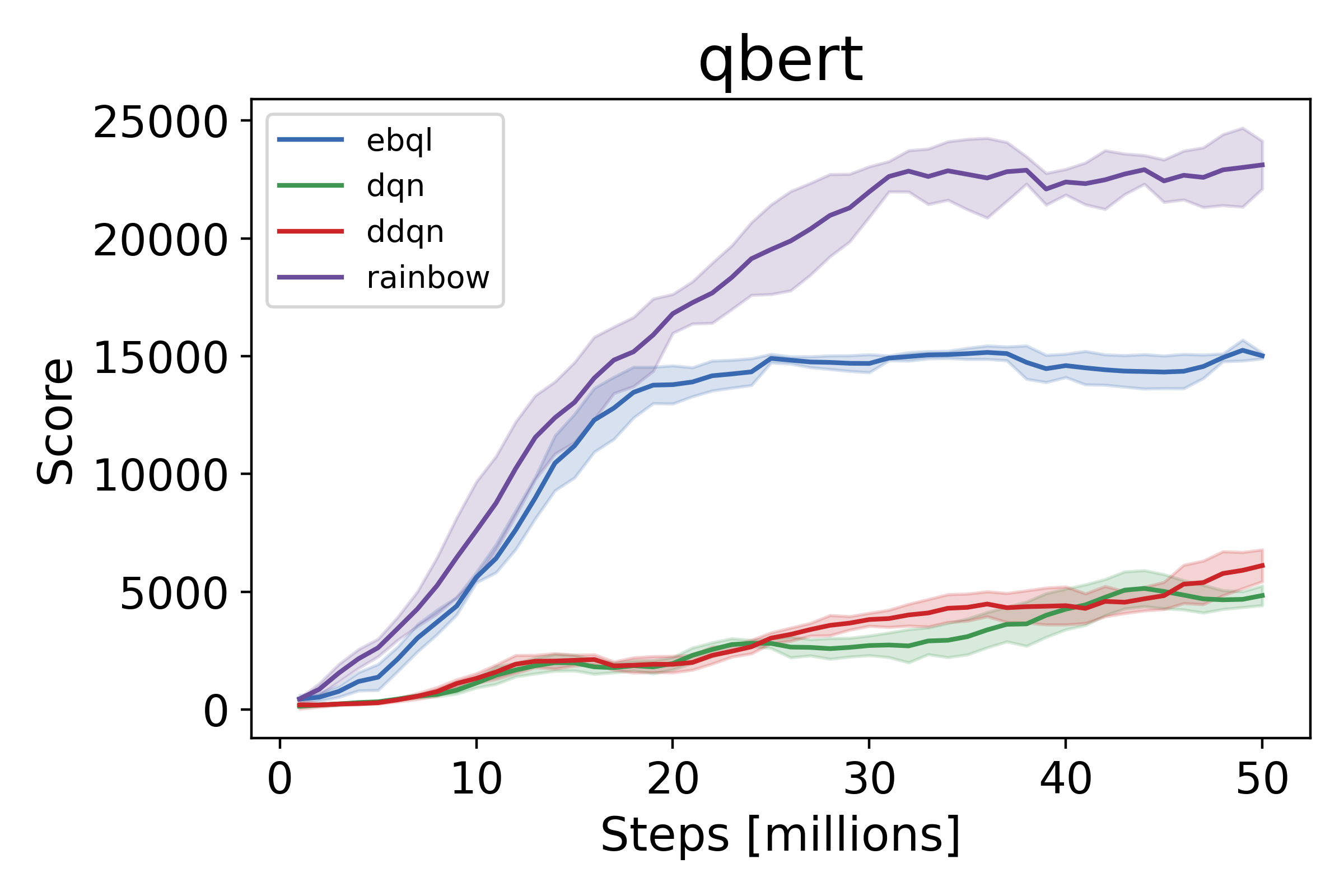}
    \includegraphics[width=0.24\linewidth]{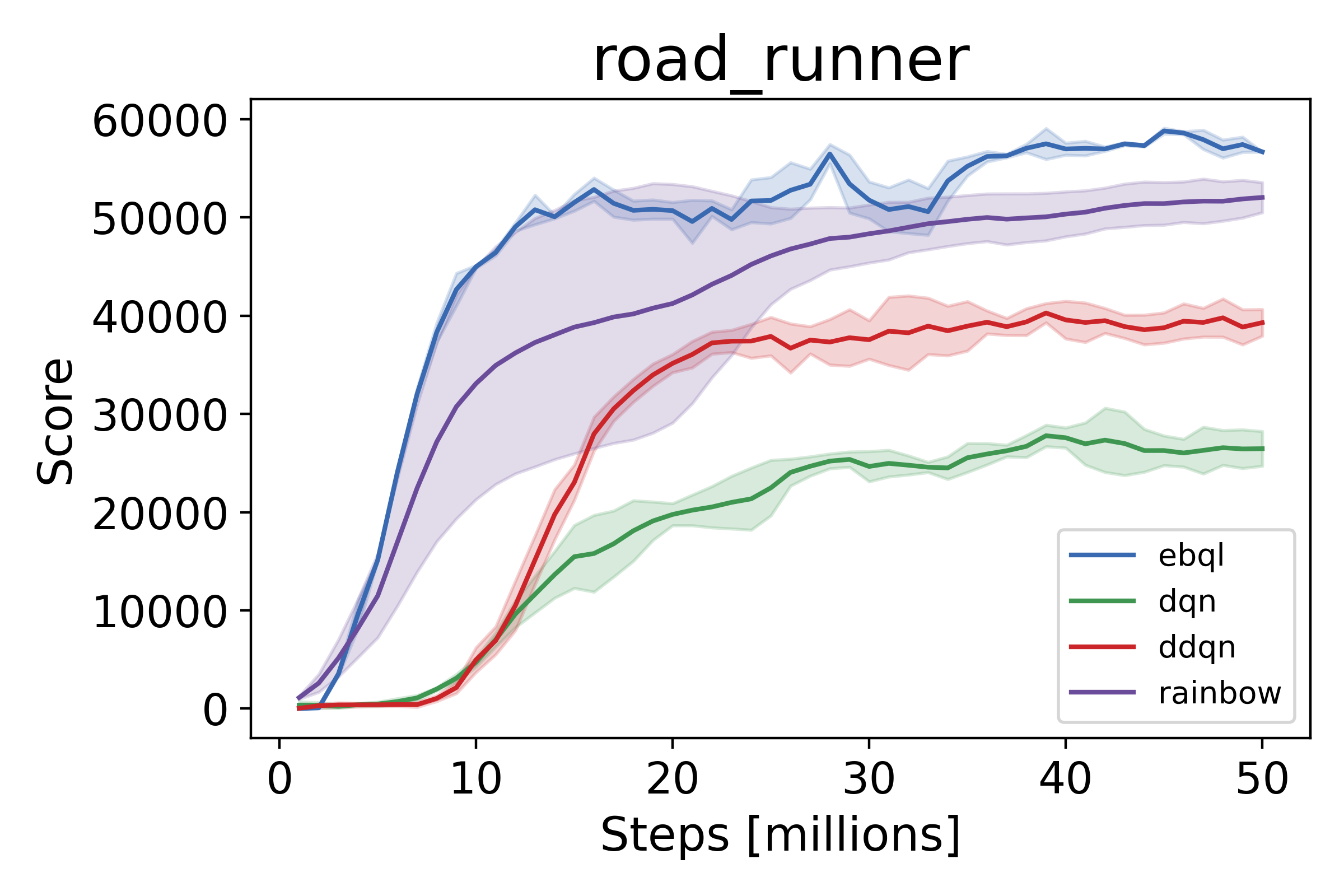}
    \includegraphics[width=0.24\linewidth]{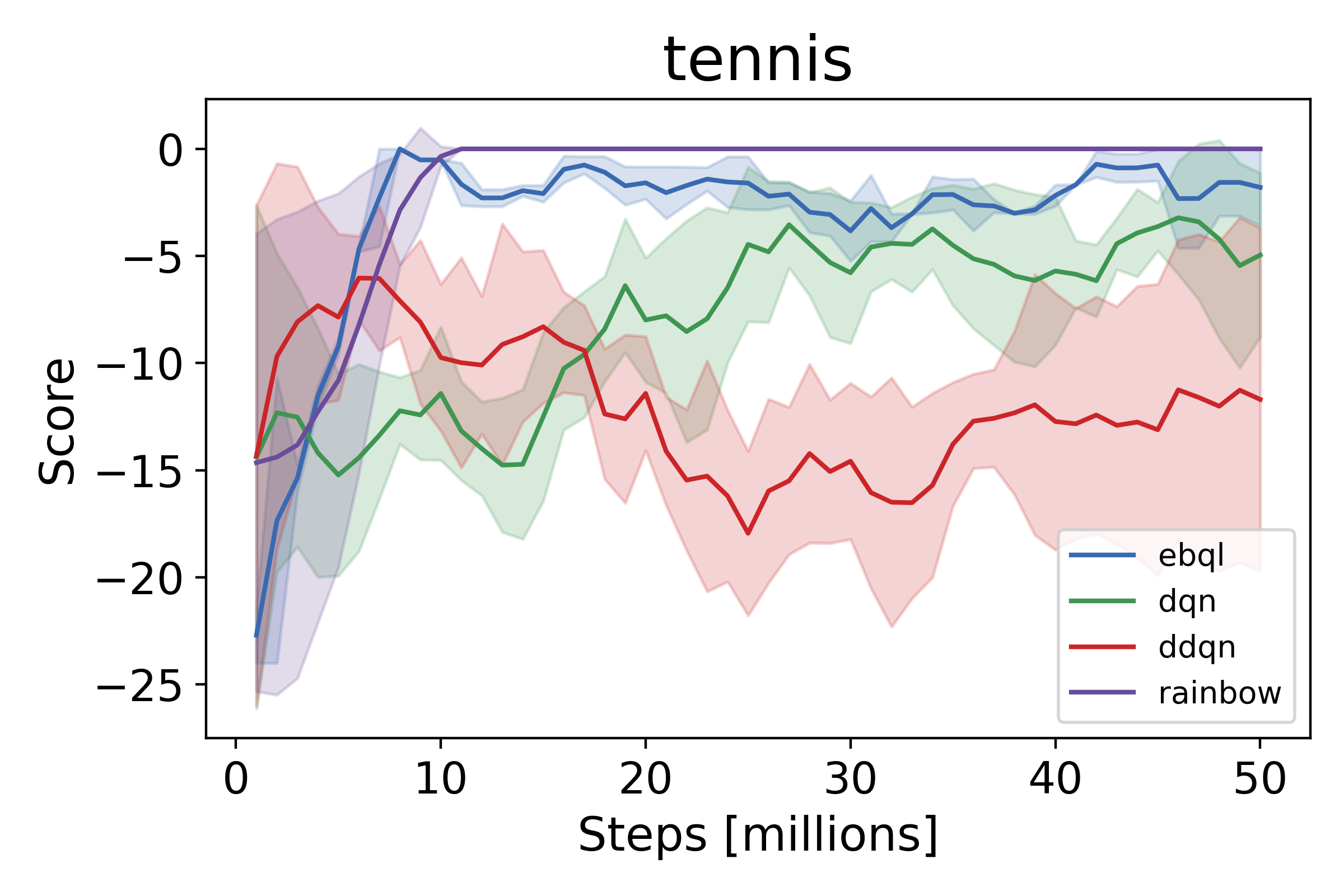}
    \includegraphics[width=0.24\linewidth]{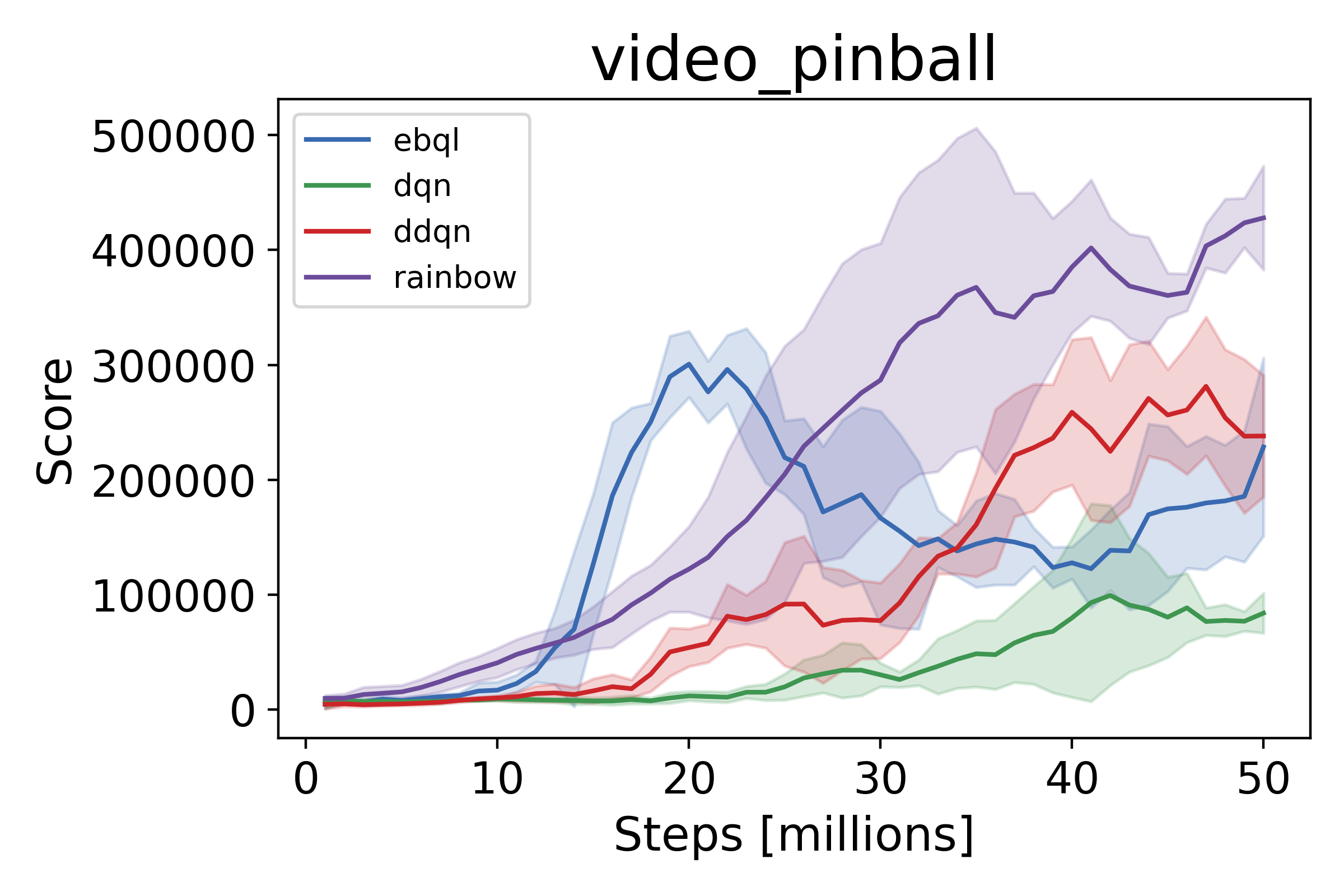}
    \caption{Comparison between DQN, DDQN, Rainbow and \alg{} on each domain over 50m steps averaged across 5 random seeds.}
    \label{fig:all envs atari boom}
\end{figure}

\end{document}